\documentclass[10pt,twocolumn,letterpaper]{article}

 \usepackage[pagenumbers]{cvpr} 

\usepackage{graphicx} 
\usepackage{amsfonts}
\usepackage{array}
\usepackage{graphicx}
\usepackage{comment}
\usepackage{algorithm}
\usepackage{algorithmic}
\usepackage{cancel}
\usepackage{booktabs}
\usepackage{enumerate}
\usepackage{amsthm}
\usepackage{amssymb}
\usepackage{makecell}

\usepackage[export]{adjustbox}
\usepackage[T1]{fontenc}

\newtheorem{theorem}{Theorem}
\usepackage{multirow} 
\usepackage[utf8]{inputenc} 
\usepackage[T1]{fontenc}    
\usepackage{url}            
\usepackage{booktabs}       
\usepackage{amsfonts}       
\usepackage{nicefrac}       
\usepackage{microtype}      
\usepackage{xcolor}         
\usepackage{enumitem}

\usepackage{lineno}
\usepackage{caption}

\usepackage{wrapfig}
\usepackage{bm}

\definecolor{cvprblue}{rgb}{0.21,0.49,0.74}
\usepackage[pagebackref,breaklinks,colorlinks,allcolors=cvprblue]{hyperref}

\def\paperID{11469} 
\def\confName{CVPR}
\def\confYear{2025}

\begin{document}

\title{Deterministic Certification of Graph Neural Networks against \\ Graph Poisoning Attacks with Arbitrary Perturbations}

\author{
Jiate Li\footnotemark[4 ], 
Meng Pang\footnotemark[2 ], Yun Dong\footnotemark[3 ], 
 and 
Binghui Wang\footnotemark[4 ] \footnotemark[1]\\ 
\footnotemark[4]\,\,{\small Department of Computer Science, Illinois Institute of Technology, Chicago, USA}\\
\footnotemark[2]\,\,{\small School of Mathematics and Computer Sciences, Nanchang University, Nanchang, China}\\
\footnotemark[3]\,\,{\small Department of Humanities, Social Science, and Communication, MSOE, Milwaukee, USA}\\
\vspace{-5mm}
}

\newcommand{\name}{\texttt{PGNNCert}}
\newcommand{\nameE}{\texttt{PGNNCert-E}}
\newcommand{\nameN}{\texttt{PGNNCert-N}}

\maketitle

\pagestyle{empty} 
\thispagestyle{empty}

\renewcommand{\thefootnote}{\fnsymbol{footnote}}
\footnotetext[1]{Corresponding author (bwang70@iit.edu)}
\renewcommand{\thefootnote}{\arabic{footnote}}

\begin{abstract}

Graph neural networks (GNNs) are becoming the de facto method to learn on the graph data and have achieved the state-of-the-art on node and graph classification tasks. 
However, recent works show GNNs are vulnerable to  training-time poisoning attacks -- marginally perturbing edges, nodes, or/and node features of training graph(s) can largely degrade GNNs' testing  performance. 
Most previous defenses 
against graph poisoning attacks 
are empirical and are soon broken by adaptive / stronger ones. 
A few provable defenses provide robustness guarantees, but have large gaps when applied in practice: 1)  restrict the attacker on only one type of perturbation; 2) design for a particular GNN architecture or task; and 3) robustness guarantees are not 100\% accurate. 

In this work, we bridge all these gaps by developing {\name}, the first certified defense of GNNs against poisoning attacks under arbitrary (edge, node, and node feature) perturbations with deterministic 
robustness guarantees. 
Extensive evaluations on multiple node and graph classification datasets and GNNs demonstrate the effectiveness of {\name} to provably defend against arbitrary poisoning perturbations. {\name} is also shown to significantly outperform the state-of-the-art certified defenses against edge perturbation or node perturbation during GNN training. 
\end{abstract}

\section{Introduction}

Graph Neural Network (GNN)~\cite{scarselli2008graph,kipf2017semi,hamilton2017inductive,velivckovic2018graph,xu2019powerful,wang2021semi}  is the leading approach for representation learning on graphs, showing state-of-the-art performance in various graph-related tasks like node classification and graph classification. In node classification, the goal is to predict labels for individual nodes, while in graph classification, the objective is to predict labels for entire graphs. GNNs have significantly advanced applications across fields such as chemistry~\cite{fung2021benchmarking}, 
physics~\cite{sanchez2020learning,shlomi2020graph},  neuroscience~\cite{bessadok2021graph},
and social science~\cite{fan2019graph}. 

However, various works~\cite{zugner2019adversarial,dai2018adversarial,wang2019attacking,xu2019topology,sun2020adversarial,zhang2021backdoor,wang2023turning,chenunderstanding,dai2023unnoticeable,ju2023let} have shown that GNNs are vulnerable to 
\emph{training-time} graph poisoning attacks
--- an attacker perturbs the graph structure during training such that the learnt poisoned GNN model will have low accuracy on predicting new test nodes/graphs. 
As a graph consists of three components: nodes, their features, and edges connecting the nodes, an attacker is allowed to perturb an individual component or their combinations. 
For instance, an attacker could inject a few nodes~\cite{sun2020adversarial,ju2023let}, slightly modify the edges~\cite{wang2019attacking,zugner2019adversarial,dai2018adversarial,xu2019topology} on the training graphs, and/or perturb features of certain nodes~\cite{zugner2019adversarial}.

Various empirical defenses \cite{zhu2019robust,tang2020transferring,entezari2020all,tao2021adversarial,zhang2020feature,zhao2021expressive} have been proposed to mitigate the graph poisoning attack, but were soon broken by adaptive attacks~\cite{mujkanovic2022defenses}. Most existing certified defenses \cite{jin2020certified,wang2021certified,jia2020certified,xia2024gnncert,lai2023nodeawarebismoothingcertifiedrobustness,li2025agnncert} are against test-time evasion attacks, with a few exceptions~\cite{jia2023pore,lai2023nodeawarebismoothingcertifiedrobustness}, leaving certified defenses against poisoning attacks largely unexplored\footnote{We note there exist some certified defense~\cite{rosenfeld2020certified, wang2020certifying,levine2020deep,jia2021intrinsic} against poisoning attacks but not for the graph data. In addition \cite{jia2023pore,lai2023nodeawarebismoothingcertifiedrobustness} show they achieve unsatisfactory performance when adapted to graph data.}. 
However, existing provable defenses face several limitations when applied in practice: 1) all restrict the attacker's capability to only one type of perturbation (e.g., node injection or edge perturbation); 2) they are designed for a particular GNN architecture or GNN task \cite{jin2020certified}; and 3) their robustness guarantee
is probabilistic (i.e., not 100\% accurate) \cite{lai2023nodeawarebismoothingcertifiedrobustness,jia2023pore}. 

We propose {\name} to address the above limitations. {\name} is the  \emph{first certified defense} for GNNs on the two most common \emph{node and graph classification tasks} against \emph{arbitrary poisoning perturbations} (i.e., arbitrarily manipulate the nodes, node features, and edges of training graph(s)) with  \emph{deterministic} robustness guarantees. 
Our defense is inspired by ensemble learning, and consists of three main steps: 
1) Divide each training graph into multiple subgraphs and allocate subgraphs of training graph(s) into multiple groups via a hash function; 2) Train a set of node/graph classifiers for each group and build a majority-voting node/graph classifier on the subgraphs; 3) Derive the deterministic robustness guarantee against arbitrary poisoning perturbations.      
{Following \cite{li2025agnncert}, we adapt} two graph division strategies---one is edge-centric and the other is node-centric---to realize our defense. The former strategy map edges, while the latter one map nodes from a given graph into multiple subgraphs.
Theoretically, {\name} provably predicts the same label for a test node/graph after training on the poisoned training set with arbitrary perturbation whose perturbation size (i.e., 
the total number of manipulated nodes, nodes with feature perturbations, and edges) is bounded by a threshold, which we call the \emph{certified perturbation size}. 
Empirically, we extensively evaluate {\name} on multiple graph datasets and multiple node and graph classifiers against arbitrary perturbations, and compared our methods with state-of-the-art certified defenses for node classification against node injection poisoning attack~\cite{lai2023nodeawarebismoothingcertifiedrobustness}, and for graph classification against edge manipulation~\cite{jia2021intrinsic}. 
Our results show {\name} significantly outperforms \cite{lai2023nodeawarebismoothingcertifiedrobustness}  under node-centric graph division, and outperforms \cite{xia2024gnncert} 
under both graph division methods. 

\noindent {\bf Contributions:} Our contributions are summarized below: 
\begin{itemize}[leftmargin=*]

\item We develop the first certified defense to robustify GNNs against arbitrary poisoning attack on the training set. 

\item We propose two strategies (edge-centric and node-centric) to realize our defense that leverages the
unique message-passing mechanism in GNNs.  

\item Our robustness guarantee is applicable to both node and graph classification tasks and accurate with probability 1. 

\item Our defense treat existing certified defenses as special cases, as well as significantly outperforming them. 
\end{itemize}
\section{Background and Problem Definition}
\label{sec:background}
\subsection{Graph Neural Network (GNN)}

Let a graph be $G=\{V,E,{\bf X}\}$, which consists of the nodes $V$, node features  ${\bf X}$, and edges $E$. We denote $u\in V$ as a node, $e=(u,v) \in E$ as an edge, and ${\bf X}_u$ as node $u$'s feature. 
 Let 
 $f_\theta$ be the node or graph classifier parameterized by $\theta$. $\mathcal{Y}$ is the label set, $y_v$ and $y_G$ are the groundtruth label of a node $v$ and a graph $G$, respectively. 

 \noindent {\bf Node classification:} 
 $f_\theta$ takes a graph $G$ as input and predicts each node $v \in G$ a label $\tilde{y}_v \in \mathcal{Y}$, i.e., $\tilde{y}_v = f_\theta(G)_v$.
Given a training node set $V_{\text{tr}} \subseteq V$ with ground-truth labels ${\bf y}_{\text{tr}} = \{y_v, v \in V_{\text{tr}}\}$, 
$f_\theta$ is learnt by minimizing a loss  $\mathcal{L}$ between the node predictions $\tilde{\bf y}_{\text{tr}}$ on $V_{\text{tr}}$ and the ground-truth ${\bf y}_{\text{tr}}$:
\begin{equation}
    {\min_\theta}\mathcal{L}({\bf y}_{\text{tr}},\tilde{\bf y}_{\text{tr}};\theta), \tilde{\bf y}_{\text{tr}}=\{f_\theta (G)_{v}, v\in V_{\text{tr}}\}
\end{equation}


 \noindent {\bf Graph classification:}  $f_\theta$ takes a graph $G$ as input and predicts a label $\tilde{y}_G \in \mathcal{Y}$ for the whole graph $G$, i.e., $\tilde{y}_G = f(G)$.
Given a set of training graphs $\mathcal{G}_{\text{tr}}$ with ground-truth labels ${\bf y}_{\text{tr}} = \{y_G, G \in \mathcal{G}_{\text{tr}} \}$. The graph classifier  $f_\theta$ is learnt by minimizing a loss function $\mathcal{L}$ between the predictions on $\mathcal{G}_{\text{tr}}$ and the ground-truth $\tilde{\bf y}_{\text{tr}}$:
\begin{equation}
    {\min_\theta} \mathcal{L}({\bf y}_{\text{tr}},\tilde{\bf y}_{\text{tr}};\theta), \tilde{\bf y}_{\text{tr}}=\{f_\theta(G), G\in \mathcal{G}_{\text{tr}}\}
\end{equation}
\subsection{Poisoning Attack on GNNs}

In poisoning attacks against GNNs, an attacker can manipulate 
any training graph $G=\{V,E,{\bf X}\} \in \mathcal{G}_{\text{tr}}$ (For node classification, $\mathcal{G}_{\text{tr}} = \{G\}$) into a perturbed one $G' = \{V',E',{\bf X}'\}$ during training, where $V'$, $E'$, ${\bf X}'$ are the perturbed version of $V$, $E$, and ${\bf X}$, respectively. For simplicity, we denote the nodes, edges and features in training graph(s) as $\mathcal{V}$, $\mathcal{E}$, and $\mathcal{X}$, respectively. 

\noindent {\bf Edge manipulation:} The attacker can 1) \emph{inject new edges} $\mathcal{E}_+$, and 2) \emph{delete existing edges}, denoted as $\mathcal{E}_-\subset \mathcal{E}$. 


\noindent {\bf Node manipulation:}
The attacker perturbs $\mathcal{G}_{\text{tr}}$ by (1) \emph{injecting new nodes} $\mathcal{V}_+$,  whose node feature denoted as $\mathcal{X}'_{\mathcal{V}_+}$ can be arbitrary, together with the arbitrarily injected new edges $\mathcal{E}_{\mathcal{V}_+} \subseteq \{(u,v) \notin \mathcal{E}, \forall u \in \mathcal{V}_+ \vee v \in \mathcal{V}_+ \}$ induced by 
$\mathcal{V}_+$; and (2) \emph{deleting existing nodes} $\mathcal{V}_- \subset \mathcal{V}$. When $\mathcal{V}_-$ are deleted, their features  $\mathcal{X}_{V_-} \subseteq \mathcal{X}$ and all connected edges  $\mathcal{E}_{\mathcal{V}_-} = \{(u,v) \in \mathcal{E}, \forall u \in \mathcal{V}_- \vee v \in \mathcal{V}_- \}$ are also removed. We denote that for the node classification case, the injected and deleted nodes are not from $V_{\text{tr}}$.


\noindent {\bf Node feature manipulation:} 
The attacker arbitrarily manipulates features $\mathcal{ X}_{\mathcal{V}_r}$ of a set of representative nodes $\mathcal{V}_{r}$ to be $\mathcal{X}'_{\mathcal{V}_r}$. 
We also denote the edges connected with nodes $\mathcal{V}_{r}$  
as $\mathcal{E}_{\mathcal{V}_r}=\{(u,v) \in \mathcal{E}: \forall u \in \mathcal{V}_{r} \vee v \in \mathcal{V}_{r} \}$.

\noindent {\bf Arbitrary manipulation:} The attacker can manipulate a training graphs in $\mathcal{G}_{\text{tr}}$ with an 
arbitrary combined perturbations on 
edges, nodes, and node features for each of them.  The attacker can manipulate several training graphs at the same time with different combinations of attack.


\emph{For description simplicity, we will use $\{\mathcal{E}_+, \mathcal{E}_-\}$ to indicate the edge manipulation with arbitrary injected edges $E_+$  and deleted edges $\mathcal{E}_-$ on $\mathcal{G}_{\text{tr}}$. Similarly, we will use $\{\mathcal{V}_+, \mathcal{E}_{\mathcal{V}_+}, \mathcal{X}'_{\mathcal{V}_+}, \mathcal{V}_-, \mathcal{E}_{\mathcal{V}_-}\}$ to indicate the node manipulation, and $\{\mathcal{V}_r, \mathcal{E}_{\mathcal{V}_r},\mathcal{X}'_{\mathcal{V}_r}\}$ the node feature manipulation. Any combination of the manipulations is inherently well-defined.} 
\subsection{Problem Statement}
\label{sec:problem}

{\bf Threat model:} Given a node/graph classifier $f$, a training graph set $\mathcal{G}_{\text{tr}}$, the adversary can \emph{arbitrarily} manipulate a number of the edges, nodes, and node features in any graph of $\mathcal{G}_{\text{tr}}$, such that after training, $f$ misclassifies target graphs in graph classification or target nodes in node classification. Since we focus on certified defenses, we consider the strongest attack where the adversary has white-box access to $\mathcal{G}_{\text{tr}}$ i.e., it knows all the edges, nodes, and node features in $\mathcal{G}_{\text{tr}}$.

\vspace{+0.05in}
\noindent {\bf Defense goal:}
We aim to build provably robust GNNs against poisoning attacks that: 
\begin{itemize}[leftmargin=*]
\item  has a deterministic robustness guarantee; 
\item is suitable for both node and graph classification tasks; 
\item provably predicts the same label against the arbitrary poisoning perturbation within the  \emph{certified perturbation size}. 
\end{itemize}

Our ultimate goal is to obtain the largest-possible certified perturbation size that satisfies all the above conditions.

\section{Our Certified Defense: {\name}}

\begin{figure*}[t]
\vspace{-4mm}
    \centering
    \captionsetup[subfloat]{labelsep=none, format=plain, labelformat=empty}

    \includegraphics[width=0.95\linewidth]{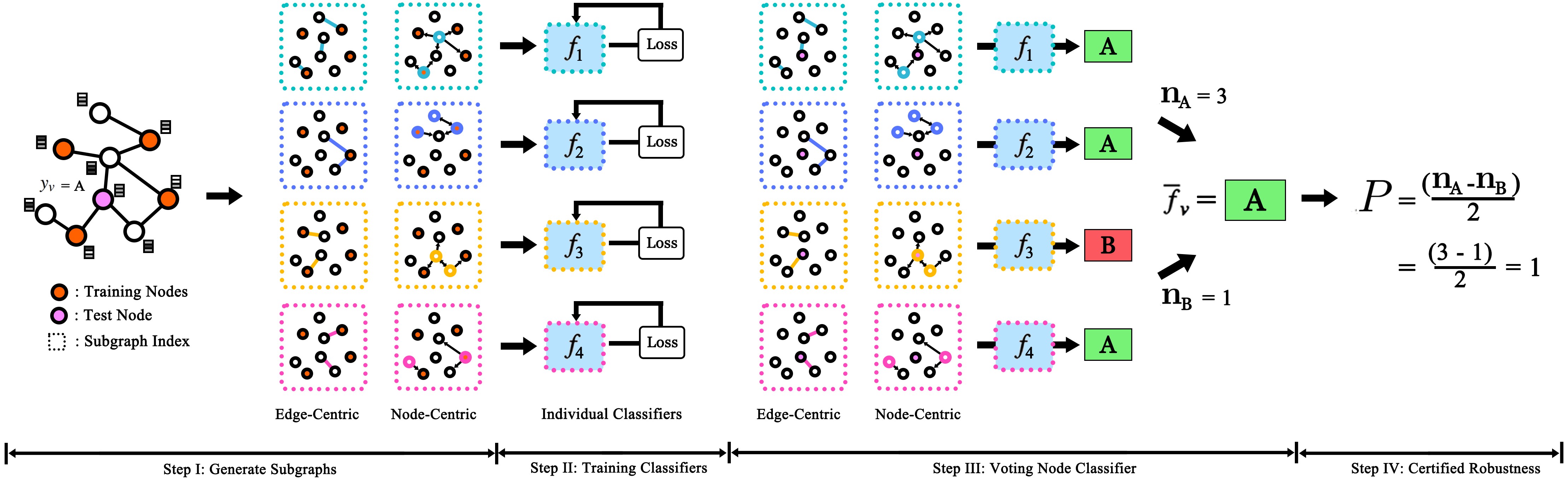}
    \vspace{-2mm}
    \caption{Overview of our {\name} (use node classification for illustration), which consists of four steps. 
    }
    \label{fig:overview}
   \vspace{-5mm}
\end{figure*}

\subsection{Overview}
\label{Sec:overview}

{Our method is inspired by previous work \cite{xia2024gnncert, li2025agnncert}, which divides a test input into several sub-parts and assembles a voting GNN classifier on the sub-parts. We generalize this idea by designing the dividing strategy tailored to training set.} 
Specifically, it consists of four steps below: 

\noindent {\bf In the first step, we divide the training graph set into multiple subgraph sets.} 
For each training graph $G\in \mathcal{G}_{\text{tr}}$ with ground-truth label ${y}_G\in \mathbf{y}_{\text{tr}}$, we divide it into $S$ subgraphs ${G}_{1},G_{2},\dots,G_{S}$ via a hash function and ensure  edges in different subgraphs are \emph{disjoint}. This process is detailed in Section \ref{Sec:E-C_Division}-\ref{Sec:N-C_Division}.
By collecting these subgraphs, we build $S$ sets of subgraphs $\mathcal{G}_{[S]}=\{\mathcal{G}_{1},\mathcal{G}_{2},\dots,\mathcal{G}_{S}\}$. In each subgraph set $\mathcal{G}_{i}$, there exists an exact subgraph $G_{i}$ generated from $G$, and we label $G_{i}$ the same label ${y}_{G}$ as $G$. 

\noindent \emph{\bf In the second step, we train multiple classifiers with the respective subgraph sets.} 
On each subgraph set $\mathcal{G}_{i}$, we initialize a classifier $f_{{i}}$ with weights $\theta_{i}$ and train it below: 

{
\vspace{-4mm}
\small
\begin{align*}
& \textbf{Node classifier: }  \min_{\theta_{i}} \mathcal{L}({\bf y}_{\text{tr}},\tilde{\bf y}_{i};\theta_{i}), \tilde{\bf y}_{i}=\{f_{i}(G)_{v}, v\in{V}_{\text{tr}}\}\\
& \textbf{Graph classifier: }  \min_{\theta_{i}} \mathcal{L}({\bf y}_{\text{tr}},\tilde{\bf y}_{i};\theta_{i}), \tilde{\bf y}_{i}=\{f_{i}(G), G\in\mathcal{G}_{\text{tr}}\}
\end{align*}
}

\noindent {\bf In the third step, we build the voting-based classifier based on the trained sub-classifiers.} 
Given a test graph $G$, we first divide it into $S$ subgraphs $\{G_{1},\dots,G_{S}\}$ by the same subgraph division method. Then we apply a voting classifier $\overline{f}$, which assembles the predictions of the trained classifiers $f_{[S]}$ on the subgraphs:

{
\vspace{-4mm}
\small
\begin{align}
& \textbf{Node classifier: } {\bf  n}_{y_v} = \sum\nolimits_{i=1}^{S}\mathbb{I}(f_{i}(G_{i})_v=y_v), \forall y_v \in \mathcal{Y} \label{eqn:vote_NC} \\
& \textbf{Graph clasifier: } {\bf  n}_{y_G} = \sum\nolimits_{i=1}^{S}\mathbb{I}(f_{i}(G_{i})=y_G), \forall y_G \in \mathcal{Y} \label{eqn:vote_GC} 
\end{align}
}

We then define our \emph{voting node/graph classifier} $\bar{f}$ as returning the class with the most vote:  

{
\vspace{-4mm}
\small
\begin{align}
& \textbf{Voting node classifier: } \bar{f}(G)_v = \underset{y_v \in \mathcal{Y}}{\arg\max} \, {\bf   n}_{y_v} \label{eqn:vc_NC} \\
& \textbf{Voting graph classifier: } \bar{f}(G) = \underset{y_G \in \mathcal{Y}}{\arg\max} \, { \bf  n}_{y_G} \label{eqn:vc_GC} 
\end{align}
}

\noindent {\bf In the forth step, we derive the deterministic robustness guarantee for the test graph.} We denote $y_{a}$ and $y_{b}$ as the class with the most vote ${\bf  n}_{y_{a}}$ and the second-most vote ${\bf n}_{y_{b}}$, respectively. We pick the class with a smaller index if ties exist. Denote $\mathcal{G}_\text{tr}'$ as the perturbed train dataset of $\mathcal{G}_\text{tr}$, and $\mathcal{G}_{[S]}'=\{\mathcal{G}_{1}', \mathcal{G}_{2}',\dots,\mathcal{G}_{S}'\}$ be the perturbed subgraph sets generated from $\mathcal{G}_{\text{tr}}'$ under the same graph division strategy. Then we have the below condition for certified robustness against arbitrary poisoning attacks on GNNs.

\begin{theorem}[Sufficient Condition for Certified Robustness]
\label{thm:suffcond}
Let $y_a, y_b, {\bf  n}_{y_a}, {\bf  n}_{y_b}$ be defined above in node classification or graph classification, and let $P = {\lfloor {\bf  n}_{y_a}-{\bf  n}_{y_b}-\mathbb{I}(y_{a}>y_{b})\rfloor} / {2}$. The voting classifier $\bar{f}$ trained on $\mathcal{G}_\text{tr}$ guarantees the same prediction on $G$ for the target node $v$ in node classification or the target graph $G$ in graph classification with the poisoned voting classifier $\bar{f}'$, if the number of different sub-classifiers (i.e., different in weights) trained on $\mathcal{G}_{[S]}$ and $\mathcal{G}_{[S]}'$ under the arbitrary perturbation is bounded by $P$. I.e., 

{
\vspace{-4mm}
\small
\begin{align}
    & \forall \mathcal{G}_{\text{tr}}': \sum\nolimits_{i=1}^{S}\mathbb{I}(\theta_{i}\neq \theta_{i}') \leq P \implies \bar{f}(G)_v = \bar{f}'(G)_v \label{eqn:suff-NC} \\ 
    & \forall \mathcal{G}_{\text{tr}}': \sum\nolimits_{i=1}^{S}\mathbb{I}(\theta_{i} \neq \theta_{i}') \leq P  \implies \bar{f}(G) = \bar{f}'(G)
    \label{eqn:suff-GC}
\end{align}
}
\end{theorem}

Proof is in Appendix \ref{supp:proofs}.The above theorem motivates us to design the graph division method such that: 1) the number of different sub-classifiers with trained on $\mathcal{G}_{[S]}$ and $\mathcal{G}_{[S]}'$ can be upper bounded (and the smaller the better).  
2) the difference between ${\bf  n}_{y_a}$ and ${\bf  n}_{y_b}$ is as large as possible, ensuring larger certified perturbation size.  

Next, we introduce our two graph division methods.  
Figure~\ref{fig:subgraphs} visualizes the divided subgraphs of the two methods without and with the adversarial manipulation.  

\subsection{Edge-Centric Graph Division}
\label{Sec:E-C_Division}

Our first graph division method is edge-centric.
The idea is to divide \emph{edges} in a graph into different subgraphs, such that each edge is deterministically mapped into \emph{only one subgraph}. 
With this strategy, we can bound the number of altered classifiers trained on these subgraphs before and after the arbitrary perturbation (Theorem~\ref{thm:edgebased}), which facilitates deriving the certified perturbation size (Theorem~\ref{thm:certifyedgebased}). 
All proofs are detailed in Appendix \ref{supp:proofs}.

\vspace{-3mm}
\subsubsection{Generating edge-centric subgraphs}
We follow \cite{xia2024gnncert,li2025agnncert} to 
use a hash function $h$ (e.g., MD5) to generate the subgraphs for every train graph $G\in \mathcal{G}_{\text{tr}}$. A hash function takes a bit string as input and outputs an integer (e.g., within a range $[0,2^{128}-1]$). We uses the string of edge or node index as the input to the hash function. 
For instance, for a node $u$,  its string is denoted as $\texttt{str}(u)$, while for an edge $e=(u,v)$, its string is $\texttt{str}(u)+\texttt{str}(v)$, where ``+" means string concatenation, and $\texttt{str}$ turns the node index into a string and adds “0” prefix to align it into a fixed length.

Assuming $S$ subgraphs in total, the subgraph index $i_e$ of every edge $e=(u,v)$ is defined as\footnote{In the undirected graph, we put the node with a smaller index (say  $u$) first and let $h[\mathrm{str}(v) + \mathrm{str}(u)]=h[\mathrm{str}(u) + \mathrm{str}(v)]$.} 

{
\vspace{-4mm}
\small
\begin{align}
\label{eqn:edgehash}
i_e = h[\mathrm{str}(u) + \mathrm{str}(v)] \, \, \texttt{mod} \, \, S+1, 
\end{align}
}
where $\texttt{mod}$ is the module function. Denoting $\mathcal{E}^i$ as the set of edges whose subgraph index is $i$, i.e., $\mathcal{E}^i = \{\forall e \in \mathcal{E}: i_e= i \}$,  $S$ subgraphs for $G$ can be built as $\{ {G}_i = (\mathcal{V},\mathcal{E}^i,{\bf X}): i=1,2,\cdots, S\}$, where edges in different subgraphs are disjoint, i.e., $\mathcal{E}^i \cap \mathcal{E}^j =  \emptyset, \forall i,j \in \{1, \cdots, S\}, i \neq j$. {Here, we mention that we need to further postprocess the subgraphs for graph classification, in order to derive the robustness guarantee. 
Particularly, in each subgraph ${G}_i$, we remove all isolated nodes who have no connected edges. This is because although they have no influence on other nodes' representation, their information would still be passed to the global graph embedding aggregation.}

After dividing all training graphs in $\mathcal{G}_{\text{tr}}$, we combine all generated subgraphs with the same index as a separate subgraph training set: $\mathcal{G}_{i}=\{G_{i}, \forall G\in \mathcal{G}_\text{tr}\}, \forall i \in [S]$.
We denote the $S$ training sets as $\mathcal{G}_{[S]} = \{\mathcal{G}_{1},\cdots, \mathcal{G}_{S}\}$.

\begin{figure*}[t]
    \centering
    \captionsetup[subfloat]{labelsep=none, format=plain, labelformat=empty}

    \subfloat[{(a) Edge-Centric Graph Division against edge injection and node injection attacks}]{
     \includegraphics[width=0.9\linewidth]{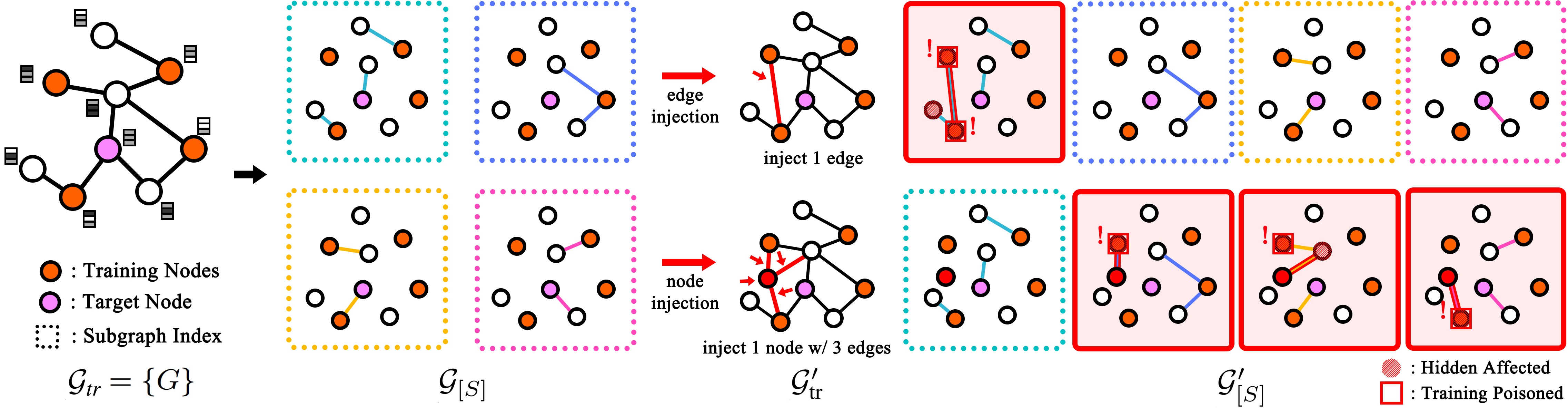}}
   
    \hspace{+20mm}
    \subfloat[{(b) Node-Centric Graph Division against edge injection and node injection attacks}]{
    \includegraphics[width=0.9\linewidth]{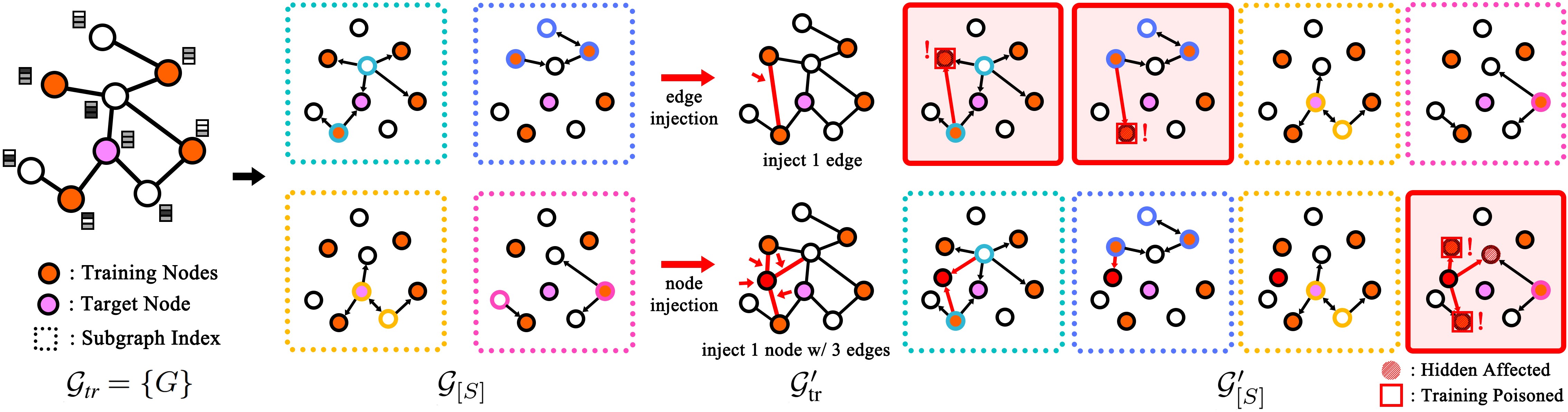}}
    
    \vspace{-2mm}
    \caption{Illustration of our edge-centric and node-centric graph division strategies for node classification. 
    We use edge injection and node injection poisoning attacks to show the bounded number of altered predictions on the generated subgraphs after the attack. 
    Figures~\ref{fig:subgraphs_NC_more}-\ref{fig:subgraphs_GC} in Appendix also show other attacks and on graph classification.}
    \label{fig:subgraphs}
    \vspace{-5mm}
\end{figure*}

\vspace{-2mm}
\subsubsection{Bounding the number of different sub-classifiers}
For a perturbed training set $\mathcal{G}'_{\text{tr}}$, we use the same graph division strategy to generate a set of $S$ subgraphs sets $\mathcal{G}'_{[S]}=\{\mathcal{G}'_1, \mathcal{G}'_2, \cdots, \mathcal{G}'_T\}$. Then, we can upper bound the number of different classifiers trained on $\mathcal{G}_{[S]}$ and $\mathcal{G}'_{[S]}$ against any individual perturbation. 

\begin{theorem}[Bounded Number of Edge-Centric Subgraphs with Altered Predictions under Arbitrary Perturbation]
\label{thm:edgebased} 
\vspace{-2mm}
Given any training graph set $\mathcal{G}_{\text{tr}}$,  
and $S$ edge-centric subgraph sets $\mathcal{G}_{[S]}$ for $\mathcal{G}_{\text{tr}}$. 
A perturbed training set $\mathcal{G}_{\text{train]}}'$ of
$\mathcal{G}_{\text{train]}}$ is 
with arbitrary edge manipulation $\{\mathcal{E}_+,\mathcal{E}_-\}$, node manipulation  
$\{\mathcal{V}_+, \mathcal{E}_{\mathcal{V}_+}, \mathcal{V}_-, \mathcal{E}_{\mathcal{V}_-}\}$, and node feature manipulation 
$\{{\bf X}_{\mathcal{V}_r}, \mathcal{V}_r, \mathcal{E}_{\mathcal{V}_r}\}$ on arbitrary graphs in $\mathcal{G}_{\text{tr}}$. 
Then at most $p=|\mathcal{E}_+| + |\mathcal{E}_-| + |\mathcal{E}_{\mathcal{V}_+}| + |\mathcal{E}_{\mathcal{V}_-}| + |\mathcal{E}_{\mathcal{V}_r}|$ 
node/graph sub-classifiers in $f_{[S]}$ are different in weights between training on the subgraph sets $\mathcal{G}_{[S]}$ and on the perturbed subgraph sets $\mathcal{G}'_{[S]}$. 
In other words, $\sum_{i=1}^{S}\mathbb{I}(\theta_{i}\neq \theta_{i}') \leq p$ for both node classification case and graph classification case. 
\vspace{-2mm}
\end{theorem}

\vspace{-3mm}
\subsubsection{Deriving the robustness guarantee}
Based on Theorems~\ref{thm:suffcond} and \ref{thm:edgebased}, we can derive the certified perturbation size as the maximal perturbation such that Equation~\ref{eqn:suff-NC} or Equation~\ref{eqn:suff-GC} is satisfied. Formally,

\begin{theorem}[Certified Robustness Guarantee with Edge-Centric Subgraphs against Arbitrary Perturbation]
\vspace{-2mm}
\label{thm:certifyedgebased} 
Let $f, y_a, y_b, {\bf  n}_{y_a}, {\bf  n}_{y_b}$ be defined above for edge-centric subgraphs, and  
$p$ be the perturbation size induced by an arbitrary perturbed training set $\mathcal{G}_{\text{tr}
}'$ on $\mathcal{G}_{\text{tr}}$. 
After training on $\mathcal{G}'_{tr}$ and $\mathcal{G}'_{tr}$, the voting classifier $\bar{f}$ and poisoned classifier $\bar{f}'$ guarantee the same prediction for the target node $v$ in node classification (i.e., $\bar{f}(G')_v = \bar{f}(G)_v$) or target graph $G$ in graph classification (i.e., $\bar{f}(G') = \bar{f}(G)$), when $p$ satisfies.
\begin{align}
\label{eqn:cpz_edge}
p \leq P = {\lfloor {\bf  n}_{y_a}-{\bf  n}_{y_b}-\mathbb{I}(y_{a}>y_{b})\rfloor} / {2}.
\end{align}
Or to say, the maximum certified perturbation size is  $P$.
\end{theorem}

\subsection{Node-Centric Graph Division}
\label{Sec:N-C_Division}

We observe the robustness guarantee under edge-centric graph division is largely dominated by the  
edges (i.e., $\mathcal{E}_{\mathcal{V}_+}, \mathcal{E}_{\mathcal{V}_-}$) induced by the manipulated nodes $\mathcal{V}_+, \mathcal{V}_-$, and edges $\mathcal{E}_{\mathcal{V}_r}$ by the perturbed node features ${\bf X}'_{\mathcal{V}_r}$. 
This guarantee could be weak against node or node feature manipulation, 
as the number of edges (i.e., $|\mathcal{E}_{\mathcal{V}_+}|, |\mathcal{E}_{\mathcal{V}_-}|, |\mathcal{E}_{\mathcal{V}_r}|$) could be much larger, compared with the number of the nodes (i.e., $|{\mathcal{V}_+}|, |{\mathcal{V}_-}|, |{\mathcal{V}_r}|$). 
For instance, an injected node could link with many edges to a given graph in practice, and when the number exceeds $P$ in Equation~\ref{eqn:cpz_edge},  the certified robustness guarantee is ineffective. 
This flaw inspires us to generate subgraphs, where 
we expect at most one subgraph is affected 
under every node or node feature manipulation on a training graph (this means all edges of a manipulated node should be in a same subgraph), implying only a training subgraph set is affected. 
{Following \cite{li2025agnncert},} we apply a tailored {node-centric graph division} strategy to achieve our goal.

\vspace{-2mm}
\subsubsection{Generating node-centric directed subgraphs}
We use a hash function $h$ to generate directed subgraphs for a given train graph $G=(\mathcal{V},\mathcal{E},{\bf X}) \in \mathcal{G}_{\text{tr}}$. 
Our node-centric graph division strategy as follow: (1) we treat every undirected edge $e=(u,v) \in G$ as two directed edges for $u$\footnote{GNNs inherently handles directed graphs with directed message passing -- each node only uses its incoming neighbors' message for update.}: the outgoing edge $u \rightarrow v$ and incoming edge $v \rightarrow u$; (2) for every node $u$, we compute the subgraph index of its every outgoing edge $u \rightarrow v$: 
{
\begin{align}
\label{eqn:nodehash}
i_{u \rightarrow v} = h[\mathrm{str}(u)] \, \, \mathrm{mod} \, \, S+1. 
\end{align}
}
Note outgoing edges of $u$ are mapped in the same subgraph.

We use $\vec{\mathcal{E}}_i$ to denote the set of directed edges whose subgraph index is $i$, i.e., $\vec{\mathcal{E}}_i = \{\forall u \rightarrow v \in {\mathcal{E}}: i_{u \rightarrow v}= i \}.$ 
Then, we can construct $S$ \emph{directed} subgraphs for $G$ as $ \vec{G}_i = (\mathcal{V},\vec{\mathcal{E}}_i,{\bf X}), \forall i\in [1,S]$.
After generating subgraphs for all training graphs, we combine all subgraphs with the same index as a separate subgraph set:
$ \vec{\mathcal{G}}_{i}=\{\vec{G}_{i}, \forall G\in \mathcal{G}_\text{tr}\}, \forall i \in [S]$.
We denote the $S$ training sets as $\vec{\mathcal{G}}_{[S]} = \{\vec{\mathcal{G}}_{1},\cdots, \vec{\mathcal{G}}_{S}\}$.
 
{Here, we mention that we need to further postprocess the subgraphs for graph classification, in order to derive the robustness guarantee. 
Particularly, in each subgraph $\vec{G}_i$, we remove all other nodes whose subgraph index is not $i$. This is because although they have no influence on other nodes' representation, their information would still be passed to the global graph embedding aggregation. To make up the loss of connectivity between nodes and simulate the aggregation, we add an extra node with a zero feature, and add an outgoing edge from every node with index $i$ to it.}

\vspace{-2mm}
\subsubsection{Bounding the number of different sub-classifiers}
For a perturbed training set ${G}'_{\text{tr}}$, we use the same graph division strategy to generate a set of $S$ \emph{directed subgraph sets} $\vec{\mathcal{G}}'_{[S]} = \{\vec{\mathcal{G}}'_1, \vec{\mathcal{G}}'_2, \cdots, \vec{\mathcal{G}}'_T\}$. 
We first show the theoretical result on bounding the number of different trained classifiers on $\vec{\mathcal{G}}_{[S]}$ and $\vec{\mathcal{G}}'_{[S]}$ against any individual perturbation.

\begin{theorem}[Bounded Number of Node-Centric Subgraphs with Altered Predictions under Arbitrary Perturbation]
\label{thm:nodebased}
\vspace{-2mm}
Let $\mathcal{G}_{\text{tr}}, v, G, \mathcal{E}_+, \mathcal{E}_-, {\mathcal{V}_+}, {\mathcal{V}_-}, \mathcal{V}_{r}$ be defined in Theorem~\ref{thm:edgebased}, and $\vec{\mathcal{G}}_{[S]}, \vec{\mathcal{G}}'_{[S]}$ contain directed subgraph sets under the node-centric graph division.  
Then, at most $\bar{p} = 2|\mathcal{E}_+|+2|\mathcal{E}_-| + |{\mathcal{V}_+}| + |{\mathcal{V}_-}| + |{\mathcal{V}_r}|$ trained node classifiers in $f_{[S]}$ are different in weights after training on $\vec{\mathcal{G}}_{[S]}$ and on $\vec{\mathcal{G}}_{[S]}'$. In other words, $\sum_{i=1}^{S}\mathbb{I}(\theta_{i} \neq \theta_{i}') \leq \bar{p}$ for any target node $v \in G$ in node classification.
Meanwhile, at most $\bar{p} = |\mathcal{E}_+|+|\mathcal{E}_-| + |{\mathcal{V}_+}| + |{\mathcal{V}_-}| + |{\mathcal{V}_r}|$ trained graph classifiers in $f_{[S]}$ are different in weights after training on $\vec{\mathcal{G}}_{[S]}$ and on $\vec{\mathcal{G}}_{[S]}'$. In other words, $\sum_{i=1}^{S}\mathbb{I}(\theta_{i} \neq \theta_{i}') \leq \bar{p}$ in graph classification.
\vspace{-2mm}
\end{theorem}

\subsubsection{Deriving the robustness guarantee}
Based on Theorem~\ref{thm:suffcond} and Theorem~\ref{thm:nodebased}, we can derive the certified perturbation size formally stated below

\begin{theorem}[Certified Robustness Guarantee with Node-Centric Subgraphs against Arbitrary Perturbation]
\label{thm:certifynodebased} 
\vspace{-2mm}
Let $f, y_a, y_b, {\bf  n}_{y_a}, {\bf  n}_{y_b}$\footnote{Note that ${\bf  n}_{y_a}, {\bf  n}_{y_b}$ have different values with those in edge-centric graph division.
Here we use the same notation for description brevity.} be defined above for node-centric subgraphs, and  
$\bar{p}$ be the perturbation size induced by an arbitrary perturbed graph $G'$ on $G$. 
With a probability 100\%, the voting classifier $\bar{f}$ guarantees the same prediction on both $G'$ and $G$ for the target node $v$ in node classification (i.e., $\bar{f}(G')_v = \bar{f}(G)_v$) or the target graph $G$ in graph classification (i.e., $\bar{f}(G') = \bar{f}(G)$),
if 
\begin{align}
\label{eqn:cpz_node}
\bar{p} \leq P = {\lfloor {\bf  n}_{y_a}-{\bf  n}_{y_b}-\mathbb{I}(y_{a}>y_{b})\rfloor} / {2}.
\end{align}
\end{theorem}
\section{Experiments}
\subsection{Experiments Settings}
\noindent {\bf Datasets:} We use four node classification datasets (Cora-ML~\cite{mccallum2000automating}, Citeseer~\cite{sen2008collective},  PubMed~\cite{sen2008collective}, Amazon-C~\cite{yang2021extract}) and four graph classification datasets (AIDS~\cite{riesen2008iam}, MUTAG~\cite{debnath1991structure}, PROTEINS~\cite{Borgwardt2005}, and DD~\cite{Dobson2003}) for evaluation. In each dataset, we take $30\%$ nodes (for node classification) or 50\% graphs (for graph classification) as the training set, $10\%$ and 20\% as the validation set and $30\%$ nodes/graph as the test set. 
Table~\ref{tab:datasets} in Appendix shows the basic  statistics of them.

\vspace{+0.05in}
\noindent {\bf GNN classifiers and {\name} training:} 
We adopt three well-known GNNs as the base node/graph classifiers: GCN~\cite{kipf2017semi}, GSAGE~\cite{hamilton2017inductive} and GAT~\cite{velivckovic2018graph}. 
We denote the two versions of {\name} under edge-centric 
and node-centric graph division as {\nameE} and {\nameN}, respectively. 
By default, we use GCN as the node/graph classifier. 

\vspace{+0.05in}
\noindent {\bf Evaluation metric:} 
Following existing works~\cite{xia2024gnncert,lai2023nodeawarebismoothingcertifiedrobustness,wang2021certified}, we use the certified node/graph  accuracy at perturbation size as the evaluation metric. 
Given a perturbation size $p$ and test nodes/graphs,  
certified node/graph accuracy at $p$ is the fraction of test nodes/graphs that are accurately classified by the voting node/graph classifier and its certified perturbation size is no smaller than $p$. 
Note that the standard node/graph accuracy is achieved when over $p=0$. 

\vspace{+0.05in}
\noindent {\bf Parameter setting:} {\name} has two hyperparameters: the hash function $h$ and the number of subgraphs $S$. By default, we use MD5 as the hash function and set $S=50, 60$ respectively for node and graph classification, considering their different graph sizes. 
 We also study the impact of them. 

\vspace{+0.05in}
\noindent {\bf Compared baselines:} 
As  {\name} encompasses existing defenses as special cases, we can compare {\name} with them against less types of perturbation. 
Here, we choose the sparse-smoothing RS~\cite{bojchevski2020efficient}, Bagging~\cite{jia2021intrinsic} and Bi-RS~\cite{lai2023nodeawarebismoothingcertifiedrobustness} as compared baselines in face of node injection poisoning attack on node classification task. For Bi-RS, We adopt the $p_{e}=0.2, p_{n}=0.9, N=1000$ setting as described, and test both include and exclude methods. We also use Bagging for comparison on graph classification.\footnote{Source code is at 
\text{https://github.com/JetRichardLee/PGNNCert}}

\begin{figure*}[!t]
\centering
\subfloat[Cora-ML]{\includegraphics[width=0.25\textwidth]{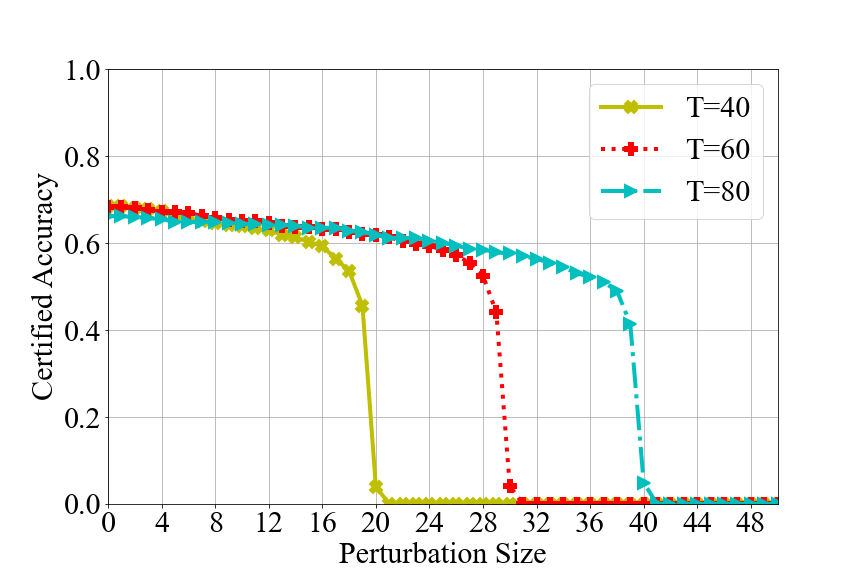}}\hfill
\subfloat[Citeseer]{\includegraphics[width=0.25\textwidth]{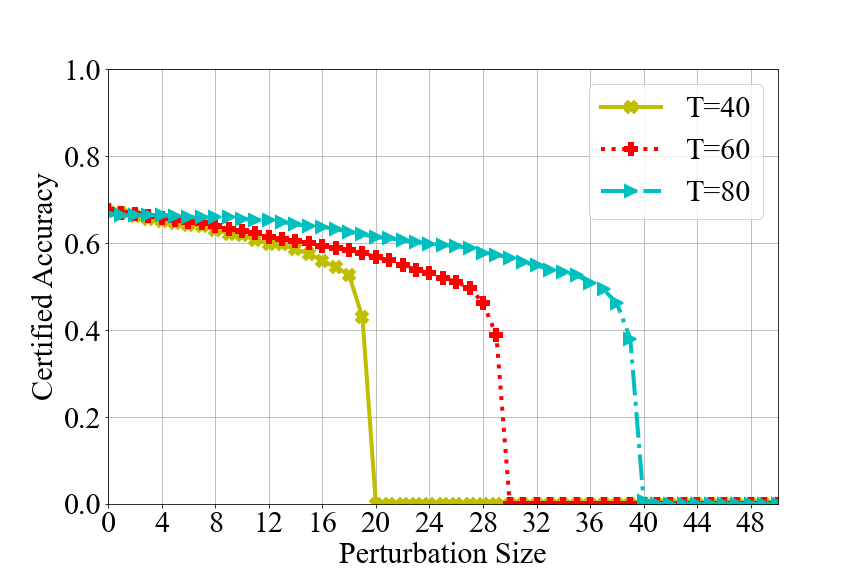}}\hfill
\subfloat[Pubmed]{\includegraphics[width=0.25\textwidth]{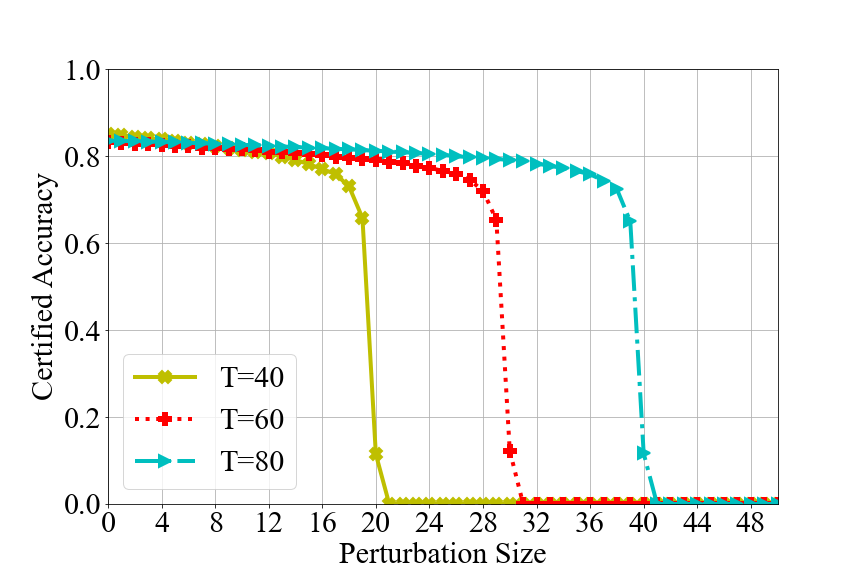}}\hfill
\subfloat[Amazon-C]{\includegraphics[width=0.25\textwidth]{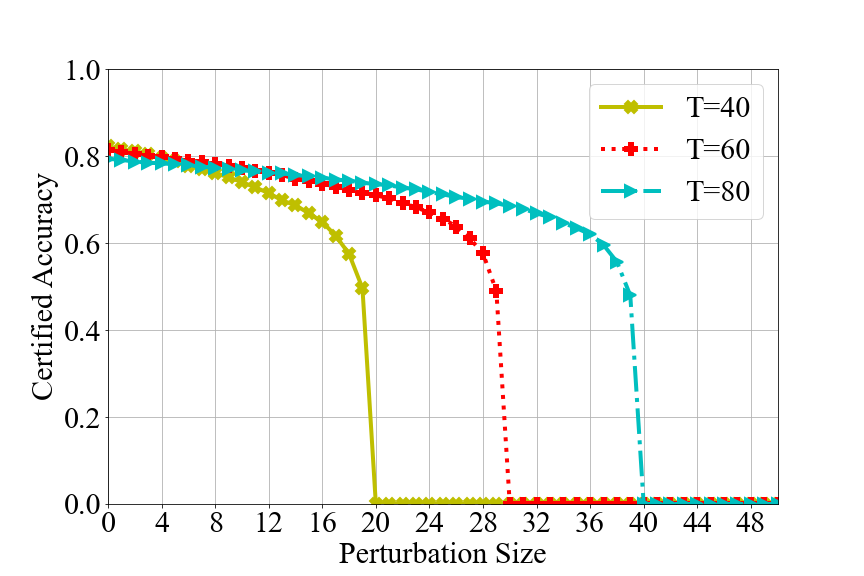}}\\
\caption{Certified node accuracy of our {\nameE} w.r.t. the number of subgraphs $S$. 
}
\label{fig:node-EC-T}
\vspace{-4mm}
\end{figure*}

\begin{figure*}[!t]
\centering
\subfloat[Cora-ML]{\includegraphics[width=0.25\textwidth]{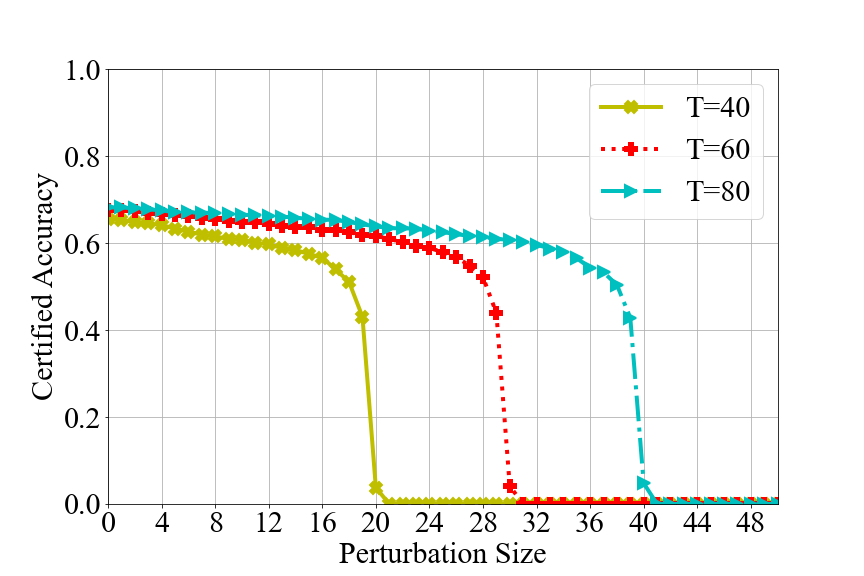}}\hfill
\subfloat[Citeseer]{\includegraphics[width=0.25\textwidth]{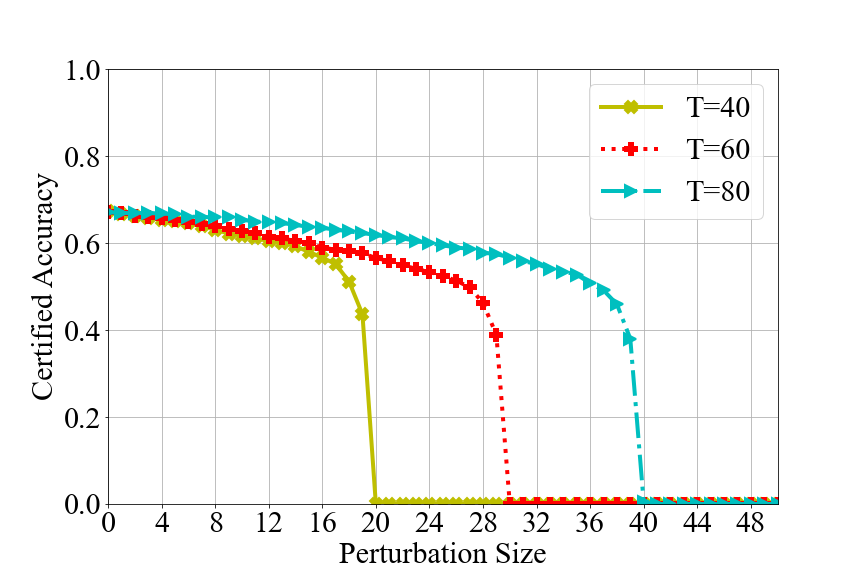}}\hfill
\subfloat[Pubmed]{\includegraphics[width=0.25\textwidth]{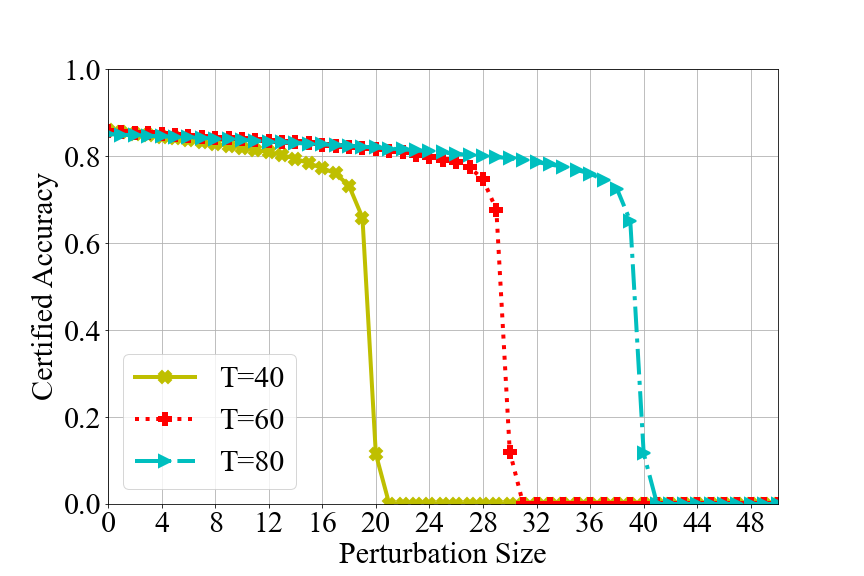}}\hfill
\subfloat[Amazon-C]{\includegraphics[width=0.25\textwidth]{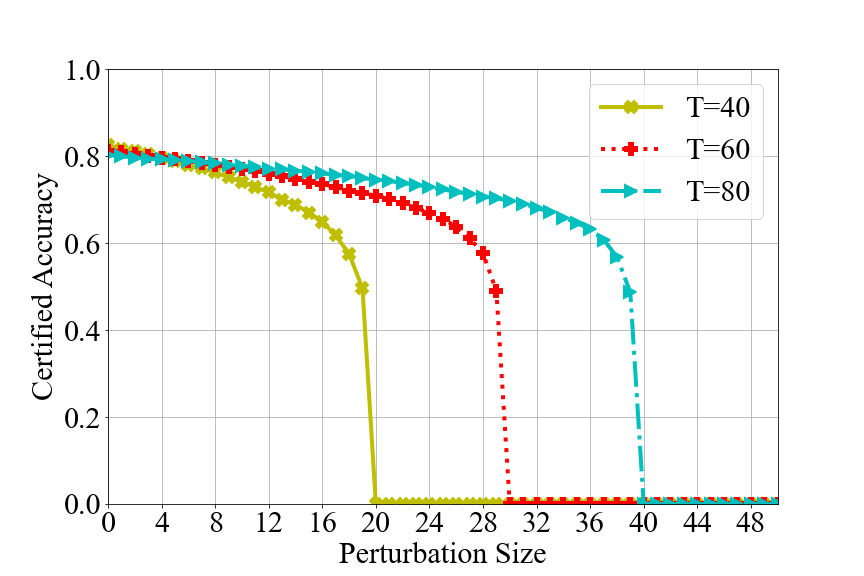}}\\
\caption{Certified node accuracy of our {\nameN} w.r.t. the number of subgraphs $S$.
}
\label{fig:node-NC-T}
\vspace{-4mm}
\end{figure*}

\begin{figure*}[!t]
\centering
\subfloat[AIDS]{\includegraphics[width=0.25\textwidth]{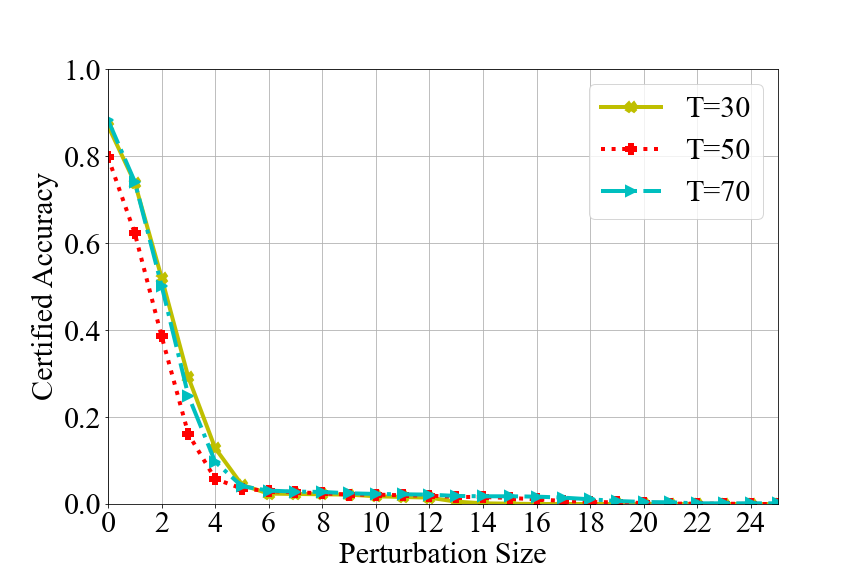}}\hfill
\subfloat[MUTAG]{\includegraphics[width=0.25\textwidth]{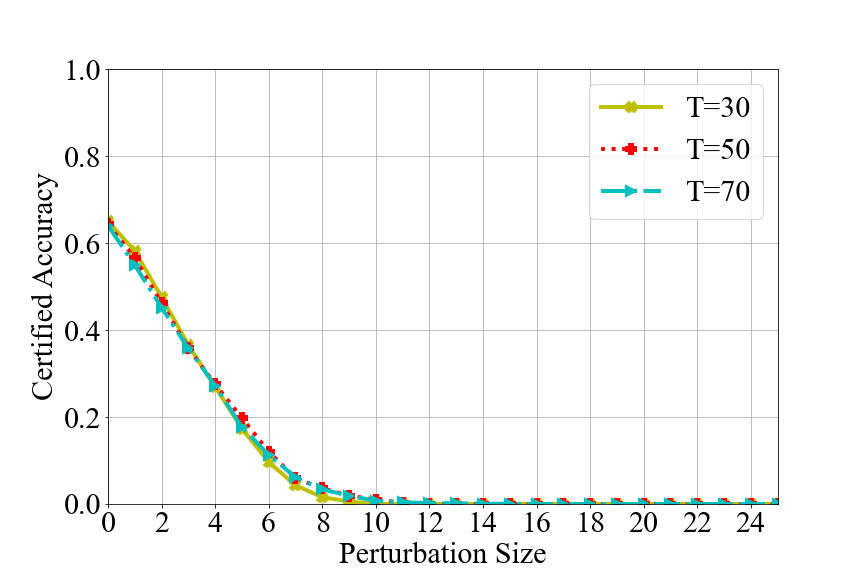}}\hfill
\subfloat[PROTEINS]{\includegraphics[width=0.25\textwidth]{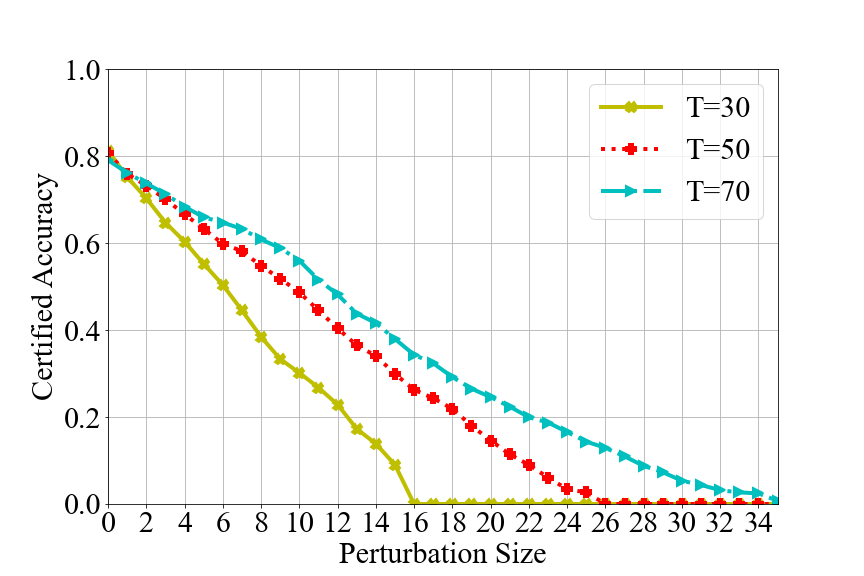}}\hfill
\subfloat[DD]{\includegraphics[width=0.25\textwidth]{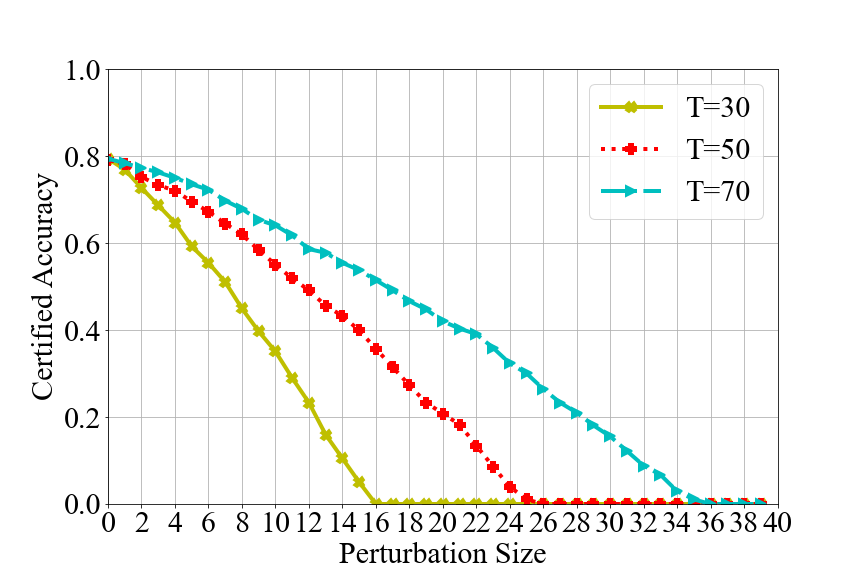}}\\
\caption{Certified graph accuracy of our {\nameE} w.r.t. the number of subgraphs $S$.
}
\label{fig:graph-EC-T}
\vspace{-4mm}
\end{figure*}

\begin{figure*}[!t]
\centering
\subfloat[AIDS]{\includegraphics[width=0.25\textwidth]{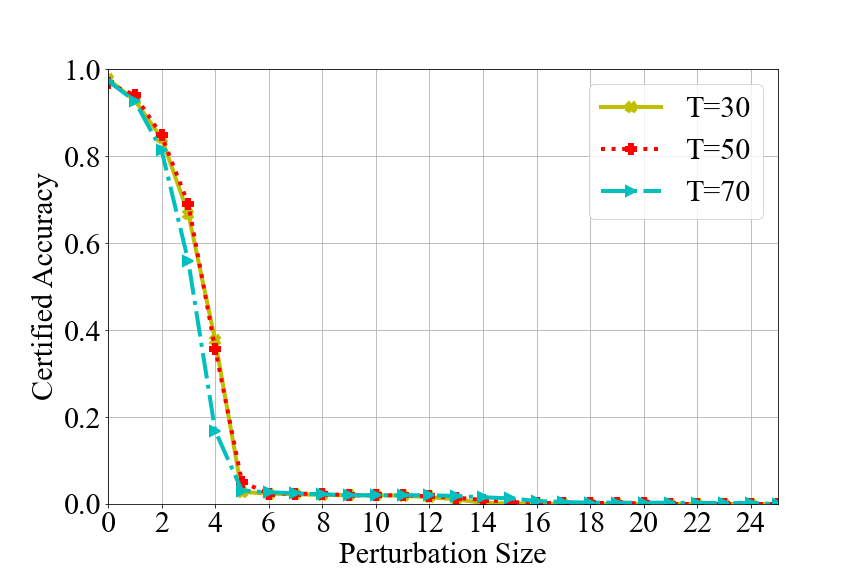}}\hfill
\subfloat[MUTAG]{\includegraphics[width=0.25\textwidth]{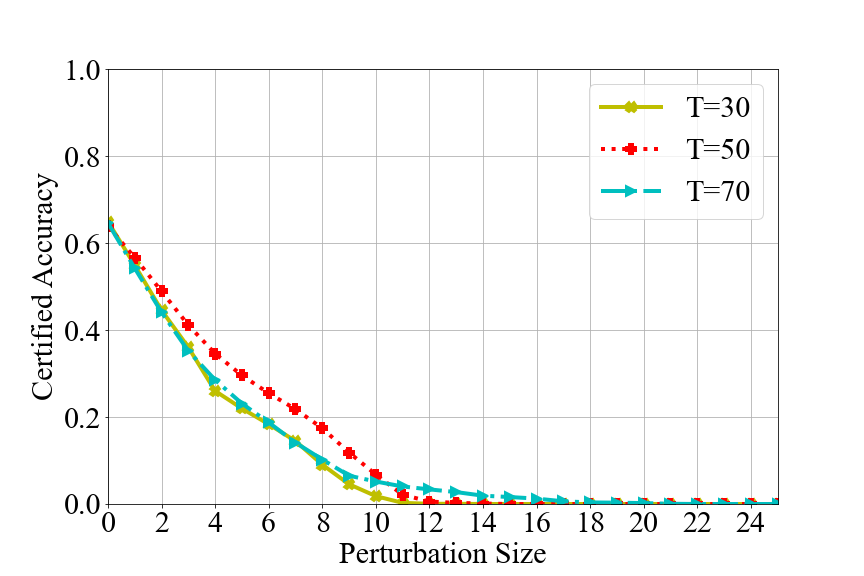}}\hfill
\subfloat[PROTEINS]{\includegraphics[width=0.25\textwidth]{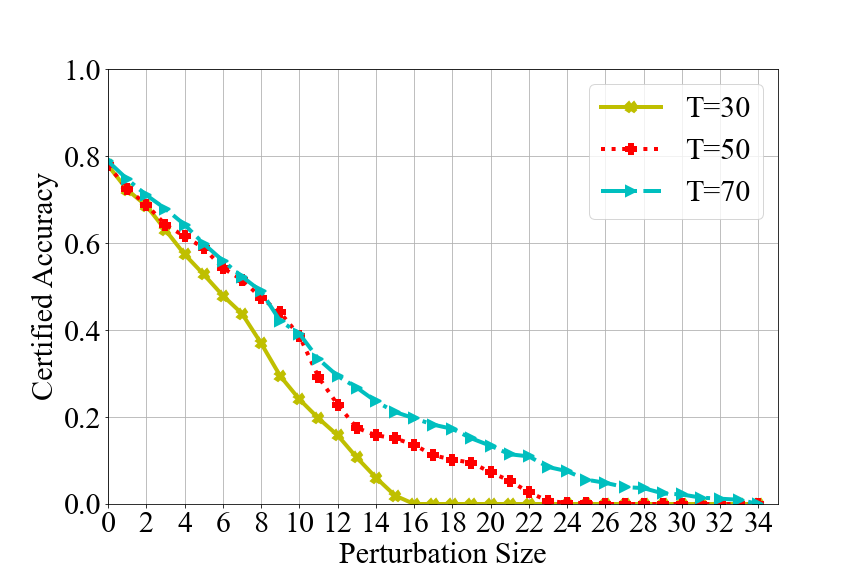}}\hfill
\subfloat[DD]{\includegraphics[width=0.25\textwidth]{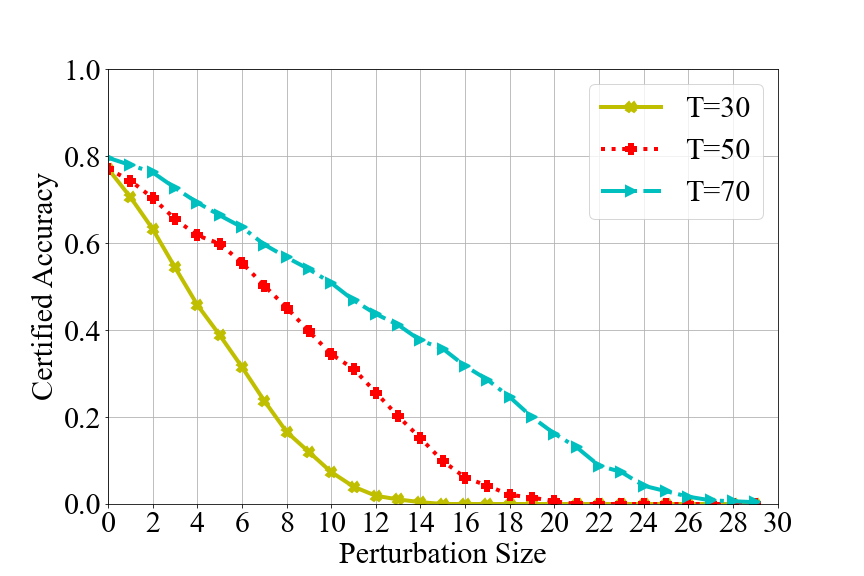}}\\
\caption{Certified graph accuracy of our {\nameN} w.r.t. the number of subgraphs $S$.}
\label{fig:graph-NC-T}
\vspace{-4mm}
\end{figure*}

\begin{figure*}[!t]
\centering
\subfloat[Cora-ML]{\includegraphics[width=0.25\textwidth]{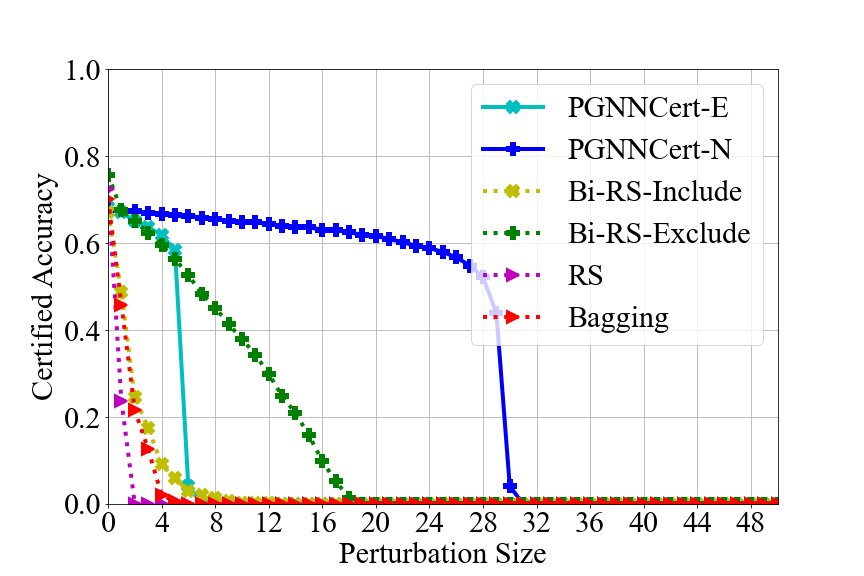}}\hfill
\subfloat[Citeseer]{\includegraphics[width=0.25\textwidth]{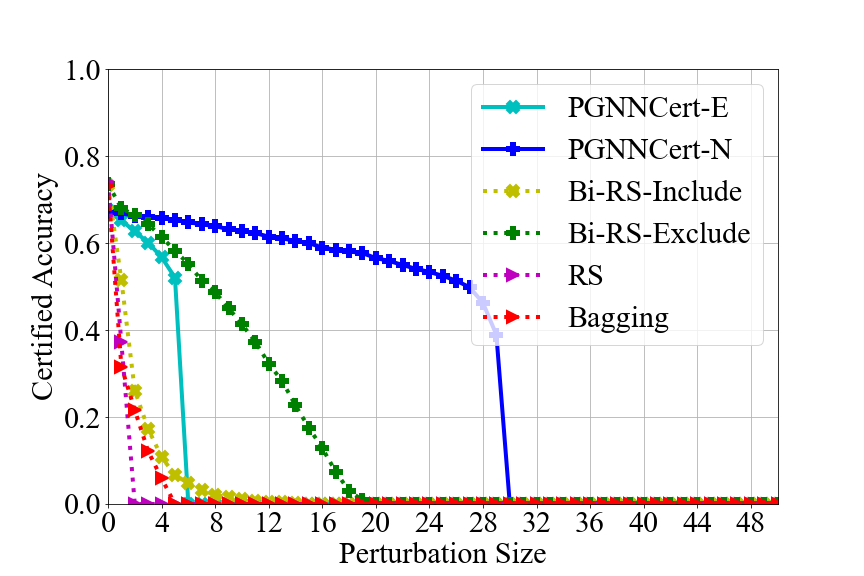}}\hfill
\subfloat[Pubmed]{\includegraphics[width=0.25\textwidth]{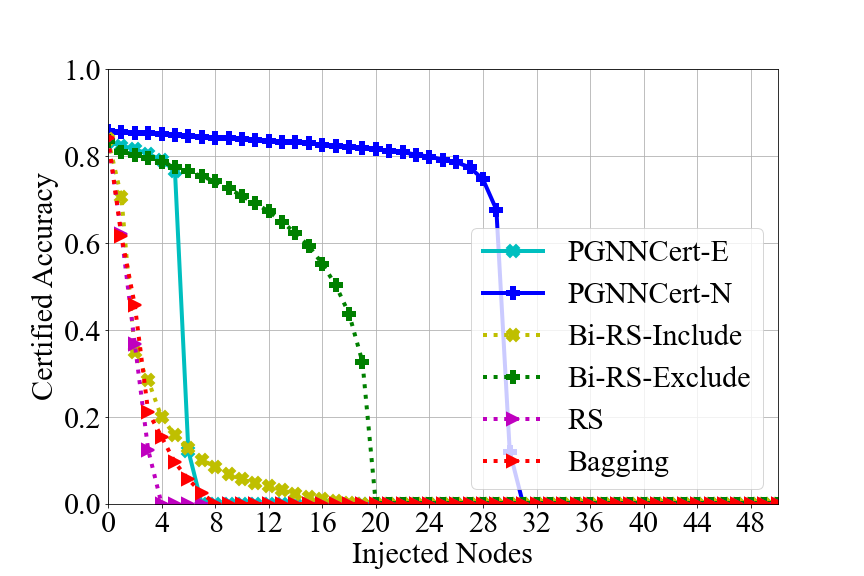}}\hfill
\subfloat[Amazon-C]{\includegraphics[width=0.25\textwidth]{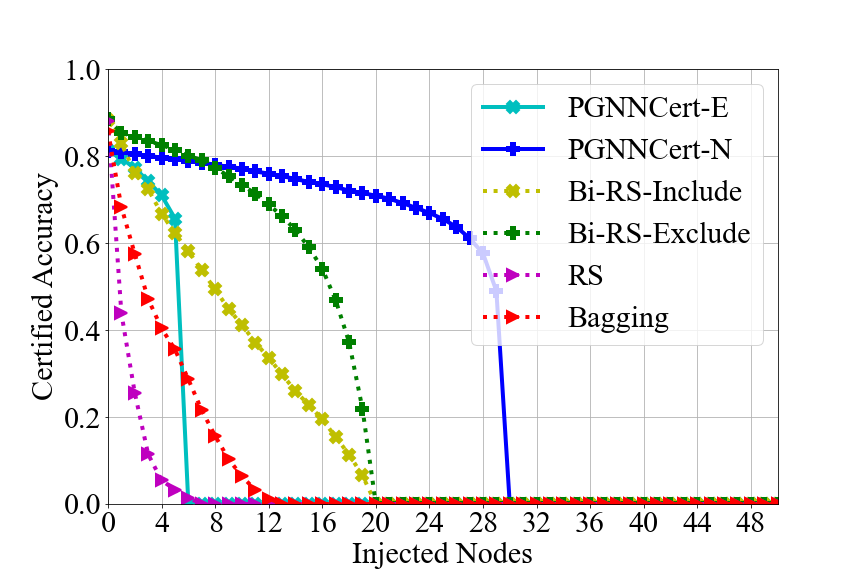}}\\
\caption{Certified node accuracy of our {\name} and SOTA defenses against node injection attacks.
}
\label{fig:node-compare}
\vspace{-3mm}
\end{figure*}

\subsection{Experiment Results on {\name}}

\noindent {\bf Main results:}
Figures~\ref{fig:node-EC-T}-\ref{fig:node-NC-T} show the certified node accuracy and Figures~\ref{fig:graph-EC-T}-\ref{fig:graph-NC-T} show the certified graph accuracy at perturbation size $p$ w.r.t. $S$ under the two graph division strategies, respectively.
We have the following observations. 

\begin{itemize}[leftmargin=*]
\item Both {\nameE} and {\nameN} can tolerate the perturbation size up to 25 and 30, on the node and graph classification datasets. This means {\nameE} can defend against a total of 25 (30) arbitrary edges, while {\nameN}  against a total of 25 (30) arbitrary edges and nodes caused by the arbitrary perturbation,  on the node (graph) classification datasets, respectively. 
Note that node classification datasets have several orders of more nodes/edges than graph classification datasets, hence {\name} can tolerate more perturbations on them.

\item $S$ acts as the robustness-accuracy tradeoff. That is, a larger (smaller) $S$ yields a higher (lower) certified perturbation size, but a smaller (higher) normal accuracy ($p=0$). 

\item In {\nameN}, the guaranteed perturbed nodes can 
have infinite edges. This implies that {\nameN} has better  robustness than {\nameE} against the perturbed  edges by node/node feature manipulation. 
\end{itemize}

\noindent {\bf Impact of base GNN classifiers:} 
Figures~\ref{fig:node-EC-T-GSAGE}-\ref{fig:node-NC-T-GSAGE} and Figures~\ref{fig:node-EC-T-GAT}-\ref{fig:node-NC-T-GAT} 
in Appendix show the certified accuracy using GSAGE and GAT as the base classifier, respectively. We have similar observations as those results with GCN.  

\vspace{+0.05in}
\noindent {\bf Impact of hash function:} 
Figures~\ref{fig:node-EC-hash}-\ref{fig:graph-NC-hash} 
show the certified node/edge accuracy of {\nameE} and  {\nameN} with different hash functions. We observe that our certified accuracy and certified perturbation size are almost the same in all cases. This reveals our {\name} is insensitive to the hash function, and \cite{xia2024gnncert} draws a similar conclusion. 

\begin{table}[!t]\renewcommand\arraystretch{0.85}
\centering
\footnotesize
\addtolength{\tabcolsep}{-3.5pt}
\caption{Node/graph accuracy of normally trained GNN and of {\name} with GNN trained on the subgraphs.}
\vspace{-2mm}
\begin{tabular}{|c|c|c|c|c|c|c|c|c|c|}
\hline
\multirow{2}{*}{\bf Dataset} & \multirow{2}{*}{\bf \scriptsize GCN} & \multicolumn{2}{c|}{\name} & \multirow{2}{*}{\bf \scriptsize GSAGE} & \multicolumn{2}{c|}{\name} & \multirow{2}{*}{\bf \scriptsize GAT} & \multicolumn{2}{c|}{\name} \\ \cline{3-4}\cline{6-7}\cline{9-10}
&&-E&-N&&-E&-N&&-E&-N
\\ \hline
{\bf \scriptsize Cora-ML} &0.73&0.68&0.68&0.67& 0.65 & 0.68 &0.74&  0.69&0.69 \\ \hline
{\bf \scriptsize Citeseer} &0.66&0.67&0.67&0.69& 0.68 & 0.68 &0.70& 0.67&0.67 
  \\ \hline
{\bf \scriptsize Pubmed} &0.86&0.83&0.85&0.84&  0.84&0.85&0.85& 0.85& 0.84 \\ \hline
{\bf \scriptsize Amazon-C}&0.81&0.81&0.81&0.82&0.78  &0.80 &0.83& 0.78& 0.80 \\ \hline \hline
{\bf \scriptsize AIDS}&0.99&0.88&0.96&0.97& 0.94 &0.95 &0.96&0.93& 0.96  \\ \hline
{\bf \scriptsize MUTAG}&0.71&0.64&0.64&0.70& 0.63 &0.64 &0.71& 0.65&0.65 \\ \hline
{\bf \scriptsize Proteins}&0.82&0.81&0.78&0.83& 0.80 &  0.81&0.82
&0.78 & 0.77\\ \hline
{\bf \scriptsize DD}&0.80 &0.79
&0.77&0.81& 0.78 & 0.76&0.81& 0.77& 0.79\\ \hline
\end{tabular}
\label{tbl:normalacc}
\vspace{-6mm}
\end{table}

\vspace{+0.05in}
\noindent {\bf Impact of subgraphs on  normal accuracy:} 
We test the normal accuracy of (not) using subgraphs to train the GNN classifier.  
Table~\ref{tbl:normalacc} shows the test node/graph accuracy of normally trained GNN without sbugraphs and {\name} with GNN trained on the subgraphs. 
We observe that the accuracy of {\name} is 5\% smaller than that of normally trained GNN in almost all cases, and in some cases even larger. This implies the augmented subgraphs for training marginally affects the normal test accuracy. 

\subsection{Comparison Results with SOTA Baselines}
In this section, we compare {\name} with SOTA defenses (Bi-RS \cite{lai2023nodeawarebismoothingcertifiedrobustness}, RS~\cite{bojchevski2020efficient}, Bagging~\cite{jia2021intrinsic}) against node injection poisoning attacks and graph structure poisoning attacks \cite{jia2021intrinsic}. 

\vspace{+0.05in}
\noindent {\bf Results for node classification against node injection poisoning attacks:} 
We follow Bi-RS~\cite{lai2023nodeawarebismoothingcertifiedrobustness} by setting the injected nodes' degrees as $\tau=5$ and show the certified node accuracy with varying number of injected nodes. Figure~\ref{fig:node-compare} shows the comparison results. We can see {\nameE} has better certified accuracy than RS, Bagging and Bi-RS-Include, but it is worse than Bi-RS-Exclude. This aligns with our assumption on its weakness against node-relevant attack. By contrast, {\nameN} outperforms all compared baselines. We highlight that each bounded node in {\nameN} can inject as many edges as possible, meaning the total number of bounded edges in {\nameN} could be infinite. 

\vspace{+0.05in}
\noindent {\bf Results for graph classification against graph poisoning attacks:} 
Table~\ref{tab:attack} shows the certified graph accuracies of {\name} and Bagging against edge perturbations on graph classification tasks. {\name} outperforms Bagging, especially when the perturbation size is larger.
    
\begin{table}[]
    \centering
    \footnotesize
    \caption{Certified graph accuracy of our {\name} and Bagging against edge perturbations.}
    \vspace{-2mm}
    \begin{tabular}{c|c|c|c|c|c} \hline
    Datasets & Methods &\#Edges=0& 5&10&15\\
    \hline
     \multirow{3}{*}{Proteins}    & Bagging&0.756&0.205&0.000&0.00 \\
         & {\nameE}&0.744&0.518&0.260 &0.00\\ 
         & {\nameN}&0.778&0.474& 0.102&0.00 \\ 
    
   \hline
     \multirow{3}{*}{DD}    & Bagging&0.785&0.311&0.000&0.00 \\
         & {\nameE}&0.792&0.578&0.330 &0.063\\ 
         & {\nameN}&0.771&0.344& 0.006&0.000 \\ \hline
    \end{tabular}
    \vspace{-2mm}
    \label{tab:attack}
\end{table}

\section{Discussion}

\noindent {\bf Evaluation on CV datasets.} We also test a benchmark superpixel graph CIFAR10  in computer vision~\cite{2023Benchmarking} for graph classification, to show the generality of our defense. CIFAR10 is an image classification dataset, which is converted into graphs using the SLIC superpixels~\cite{2012SLIC}. Each node has the superpixel coordinates as the feature and each superpixel graph has a label. Results in the default setting are shown in Table \ref{tab:cvresults}, where we see {\name} can obtain about 50\% certified accuracy when 3 arbitrary edges are perturbed.  

\begin{table}[!t]

    \footnotesize
    \centering
    \caption{Results of {\name} on CIFAR10.}
    \vspace{-2mm}
    \begin{tabular}{c|c|c|c|c|c}
    \hline
    Datasets&Methods&p=0&1&3&5\\
    \hline
         \multirow{2}{*}{CIFAR10}&\nameE&0.596&0.556&0.482&0.107  \\
         &\nameN&0.636&0.593&0.498&0.113\\ 
         \hline
    \end{tabular}
    \vspace{-2mm}
    
    \label{tab:cvresults}
    \end{table}

\vspace{+0.05in}
\noindent {\bf {\name} against optimization-based attacks.} 
We first emphasize that {\name} is provably robust against {\emph{all}} (known and unknown)  attacks, when their perturbation budget is within the derived bound (in Eqn \ref{eqn:cpz_edge} or Eqn \ref{eqn:cpz_node}), regardless of the attack knowledge of \name. 
Here, we test {\nameN} against  Metattack \cite{zugner2019adversarial}  when its perturbation size $m$ is larger than the derived certified perturbation size {P=25 (with certified accuracy 0.588)} on Cora-ML with S=60. Results are shown in Table \ref{tab:moreattack}. We see {\nameN} can achieve {0.661} accuracy of defending against Metattack even when it perturbs 40 edges and nodes in total.

\begin{table}[!t]
    \footnotesize
    \caption{Defense results against the optimization-based Mettack.}
    \vspace{-2mm}
    \centering
    \begin{tabular}{c|c|c|c|c|c}
    \hline
    {Dataset}&{Nettack}&p=0&20&30&40\\
    \hline
         {Cora-ML}
         &\nameN&0.675&0.675&0.668&0.661\\ 
         \hline
    \end{tabular}
     
    \label{tab:moreattack}
    \vspace{-4mm}
\end{table}

\vspace{+0.05in}
\noindent {\bf Limitations.} Despite the effectiveness of {\name}, its inefficiency remains a main limitation for large GNNs. 
{\name} trains $S$ classifiers on $S$ subgraphs and uses $S$ subgraphs for certification. It may incur a training/certification complexity that is $S$ times of standard GNN training/testing. This overhead could be significant when {\name} is applied to large GNNs. On the other hand, we note that the $S$ classifiers can be trained in parallel once obtaining the $S$ subgraphs after graph division. 
\label{sec:discussion}
\section{Related Work}
\noindent {\bf Adversarial attacks on GNNs:} 
Existing attacks to GNNs can be classified as graph {evasion attacks}~\cite{dai2018adversarial,zugner2018adversarial,xu2019topology,wu2019adversarial,ma2019attacking,ma2020towards,mu2021a,wang2022bandits,wang2023turning,wang2024efficient,li2024graph,li2025practicable} and {poisoning attacks}~\cite{zugner2019adversarial,dai2018adversarial,zugner2019adversarial,xu2019topology,wang2019attacking,takahashi2019indirect,liu2019unified,sun2020adversarial,zhang2021backdoor}.  
For instance, \cite{dai2018adversarial} leveraged reinforcement learning techniques to design evasion attacks to both graph classification and node classification via modifying the graph structure. 
Most attacks require the attacker fully or partially knows the GNN model, while \cite{mu2021a,wang2022bandits} relaxing this to  only have black-box access, i.e., only query the GNN model API. For example,  \cite{wang2022bandits} formulates the black-box attack to GNNs as an online optimization problem with bandit feedback and provably obtains a sublinear query number. 
Furthermore, \cite{li2024graph} generalizes the black-box evasion attack on explainable GNNs.  
In the realm of graph poisoning attacks, \cite{xu2019topology} proposed a MinMax attack to GNNs based on gradient-based optimization, while \cite{zugner2019adversarial} introduced Metattack that perturbs the entire graph using meta learning. 

\vspace{+0.05in}
\noindent {\bf Certified robustness against graph poisoning graphs:} 
Most existing certified defenses for GNNs  are against test-time evasion attacks \cite{jin2020certified,jia2020certified,wang2021certified,zhang2021backdoor,xia2024gnncert,lai2023nodeawarebismoothingcertifiedrobustness,li2025provably,li2025agnncert},  
with a few ones \cite{lai2023nodeawarebismoothingcertifiedrobustness} against training-time poisoning attacks. 

There exist two main approaches providing certified robustness against poisoning data: 1) Randomized-smoothing based methods~\cite{lai2023nodeawarebismoothingcertifiedrobustness,rosenfeld2020certified,wang2020certifying} regard the training-prediction process as an end-to-end function, and treat poisoning attack as a special case of evasion attack; 2) Voting based methods~\cite{levine2020deep,jia2021intrinsic} partition training data into subsets and train a sub-classifier on each subset. 
However, they restrict the attacker on one type of perturbation (nodes, edges, or node features); they are applicable for a particular GNN task; and their robustness guarantees are not 100\% accurate. 
We are the first work to develop a deterministically certified robust GNN against graph poisoning attack with arbitrary kind of perturbations on both node and graph classification tasks.

\section{Conclusion}

We investigate the robustness of Graph Neural Networks (GNNs) against graph poisoning attacks and introduce {\name}, the first certified defense with deterministic guarantees against arbitrary poisoning perturbations, including modifications to nodes, edges, and node features. {\name} employs novel graph division strategies and leverages the message-passing mechanism in GNNs to establish robustness guarantees. Its universality allows it to encompass existing certified defenses as special cases. Experimental evaluations demonstrate that {\name} effectively mitigates arbitrary poisoning perturbations, offering superior robustness and efficiency compared to state-of-the-art certified defenses.
Future works include extending the proposed defense for \emph{federated} GNNs \cite{wang2022graphfl,yang2024distributed} and \emph{casually explainable} GNNs \cite{behnam2024graph} against arbitrary poisoning attacks. 

\newpage

\section{Acknowledgment}
We thank all the anonymous reviewers for their valuable feedback and constructive comments. 
This work is partially supported by the National Science Foundation (NSF) under grant Nos. ECCS-2216926, CNS-2241713, CNS-2331302, CNS-2339686, and the Cisco Research Award.

{
    \small    \bibliographystyle{ieeenat_fullname}
    \bibliography{main}
}

\appendix
\clearpage
\setcounter{page}{1}
\maketitlesupplementary

\section{Proofs}
\label{supp:proofs}

{ 

\subsection{Proof of Theorem~\ref{thm:suffcond}}
\label{app:suffcond}

We prove for node classification and it is identical for graph classification. 

Recall $y_a$ and $y_b$ are respectively the class with the most vote ${\bf n}_{y_a}$ and with the second-most vote ${\bf n}_{y_b}$ on predicting the target node $v$ in the subgraphs $\{G_i\}'s$. Hence, 
\begin{align}
&    {\bf n}_{y_{a}}-\mathbb{I}(y_{a}>y_{b})\geq {\bf n}_{y_{b}} \label{eqn:16} \\
 &   {\bf n}_{y_{b}}-\mathbb{I}(y_{b}>y_{c})\geq {\bf n}_{y_{c}}, \forall y_c \in \mathcal{Y}\setminus\{y_{a}\} \label{eqn:17}
\end{align}
where $\mathbb{I}$ is the indicator function, and we pick the class with a smaller index when there exist ties. 

Further, on the poisoned classifiers $f'_{[S]}$ with $\theta'_{[S]}$ after the attack, the vote ${\bf n}_{y_{a}}'$ of the class $y_a$ and vote ${\bf n}_{y_{c}}'$ of any other class $y_{c}\in \mathcal{Y}\setminus \{y_{a}\}$ satisfy the below relationship: 
\begin{equation}
\label{eqn:18}
{\bf n}_{y_{a}}'\geq {\bf n}_{y_{a}} - \sum_{i=1}^{T}\mathbb{I}(f_{i}(G_{i})_v\neq f'_{i}(G_{i})_v) 
\end{equation}
\begin{equation}
\label{eqn:19}
{\bf n}_{y_{c}}'\leq {\bf n}_{y_{c}} + \sum_{i=1}^{T}\mathbb{I}(f_{i}(G_{i})_v\neq f'_{i}(G_{i})_v)
\end{equation}

Since $f_{[S]}$ and $f'_{[S]}$ only differ in trained weights, the above expression $\sum_{i=1}^{T}\mathbb{I}(f_{i}(G_{i})_v\neq f'_{i}(G_{i})_v)$ could be replaced by $\sum_{i=1}^{T}\mathbb{I}(\theta_{i}\neq \theta'_{i})$

To ensure the returned label by the voting node classifier $\bar{f}$ does not change, i.e., $\bar{f}(G)_v = \bar{f}'(G)_v, \forall \mathcal{G}'_{tr}$, we must have:
\begin{equation}
\label{eqn:20}
{\bf n}_{y_{a}}'\geq {\bf n}_{y_c}'+\mathbb{I}(y_{a}>y_{c}),\forall y_{c}\in \mathcal{Y}\setminus \{y_{a}\}
\end{equation}

Combining with Eqns \ref{eqn:18} and \ref{eqn:19}, the sufficient condition for Eqn~\ref{eqn:20} to satisfy is to ensure: 
\begin{equation}
{\bf n}_{y_{a}} - \sum_{i=1}^{T}\mathbb{I}(\theta_{i}\neq \theta'_{i})  \geq 
{\bf n}_{y_{c}} + \sum_{i=1}^{T}\mathbb{I}(\theta_{i}\neq \theta'_{i})
\end{equation}
Or, 
\begin{equation}
{\bf n}_{y_{a}}\geq {\bf n}_{y_{c}} + 2\sum_{i=1}^{T}\mathbb{I}(\theta_{i}\neq \theta'_{i})+\mathbb{I}(y_{a}>y_{c}).
\end{equation}

Plugging Eqn~\ref{eqn:17}, we further have this condition:
\begin{equation}
\label{eqn:23}
{\bf n}_{y_{a}} \geq {\bf n}_{y_{b}}-\mathbb{I}(y_{b}>y_{c})+ 2\sum_{i=1}^{T}\mathbb{I}(\theta_{i}\neq \theta'_{i})+\mathbb{I}(y_{a}>y_{c})
\end{equation}
We observe that:
\begin{equation}
\label{eqn:24}
\mathbb{I}(y_{a}>y_{b})\geq \mathbb{I}(y_{a}>y_{c})-\mathbb{I}(y_{b}>y_{c})
,\forall y_{c}\in \mathcal{Y}\setminus \{y_{a}\}\end{equation}
Combining Eqn~\ref{eqn:24} with Eqn~\ref{eqn:23}, we have:
\begin{equation}
{\bf n}_{y_{a}} \geq {\bf n}_{y_{b}}+2\sum_{i=1}^{T}\mathbb{I}(\theta_{i}\neq \theta'_{i})+\mathbb{I}(y_{a}>y_{b})
\end{equation}
Let $M = {\lfloor {\bf n}_{y_a}-{\bf n}_{y_b}-\mathbb{I}(y_{a}>y_{b})\rfloor} / {2}$, hence 
$$\sum\nolimits_{i=1}^{T}\mathbb{I}(\theta_{i}\neq \theta'_{i}) \leq M.$$

\subsection{Proof of Theorem 2}
To prove Theorem~\ref{thm:edgebased}, we will first certify the bounded number of altered predictions under (1) edge manipulation, (2) node manipulation and (3) node feature manipulation separately through Theorems~\ref{thm:edgeperturb}-\ref{thm:nodefeaperturb}. 
\begin{theorem}[]
\label{thm:edgeperturb}
Assume $\mathcal{G}_{\text{tr}}$ is under the edge manipulation $\{\mathcal{E}_+,\mathcal{E}_-\}$, 
then at most $|\mathcal{E}_+| + |\mathcal{E}_-|$ sub-classifiers trained by our edge-centric subgraph sets are different between $\mathcal{G}'_{[S]}$ and $\mathcal{G}_{[S]}$. 
\end{theorem}

\begin{proof}
Edges of a train graph $G$ in all subgraph sets of $\mathcal{G}_{[S]}$ are disjoint. 
Hence, when any edge in $G$ is deleted or added by an adversary, only one subgraph set in $\mathcal{G}_{[S]}$ is affected. Further, when any $|\mathcal{E}_+| + |\mathcal{E}_-|$ edges in $G$ are perturbed, there are at most $|\mathcal{E}_+| + |\mathcal{E}_-|$ subgraph set between $\mathcal{G}_{[S]}$ and $\mathcal{G}'_{[S]}$ are different.
 By training $S$ node/graph sub-classifiers on $\mathcal{G}_{[S]}$ and $\mathcal{G}'_{[S]}$, there are at most $|\mathcal{E}_+| + |\mathcal{E}_-|$ sub-classifiers that have different weights between them.  
\end{proof}

\begin{theorem}[]
\label{thm:nodeperturb} 
Assume the training graph set $\mathcal{G}_{\text{tr}}$ is under the node manipulation  
$\{\mathcal{V}_+, \mathcal{E}_{\mathcal{V}_+},{\bf X}'_{\mathcal{V}_+},\mathcal{V}_-, \mathcal{E}_{\mathcal{V}_-}\}$,
then at most $|\mathcal{E}_{\mathcal{V}_+}| + | \mathcal{E}_{\mathcal{V}_-}| $ sub-classifiers trained by our edge-centric subgraph sets are different between $\mathcal{G}'_{[S]}$ and $\mathcal{G}_{[S]}$. 
\end{theorem}

\begin{theorem}[]
\label{thm:nodefeaperturb} 
Assume the training graph set $\mathcal{G}_{\text{tr}}$ is under the node feature manipulation 
$\{\mathcal{V}_r, \mathcal{E}_{\mathcal{V}_r},{\bf X}'_{\mathcal{V}_r}\}$, 
then at most $|\mathcal{E}_{\mathcal{V}_r}|$ sub-classifiers trained by our edge-centric subgraph sets are different between $\mathcal{G}'_{[S]}$ and $\mathcal{G}_{[S]}$. 
\end{theorem}

\begin{proof}
{ 
Our proof for the above two theorems is based on the key observation that manipulations on isolated nodes do not participate in the forward calculation of other nodes' representations in GNNs. 
Take node injection for instance and the proof for other cases are similar.  
Note that all subgraphs after node injection will contain the newly injected nodes, but they still do not have overlapped edges between each other via the hash mapping. Hence, the edges $E_{\mathcal{V}_+}$ induced by the injected nodes $\mathcal{V}_+$ exist in at most $|E_{\mathcal{V}_+}|$ subgraphs. In other word, the injected nodes $\mathcal{V}_+$ in at least $S-|E_{+}|$ subgraphs have no  edges and are isolated. 

Due to the message passing mechanism in GNNs, every node only uses its neighboring nodes' representations to update its own representation. Hence, 
 these subgraphs with 
the isolated injected nodes, whatever their features ${\bf X}'_{\mathcal{V}_+}$ are, would have no influence on other nodes' representation calculation. Therefore, in at least $S-|E_{+}|$ subgraph sets, the training nodes'/graphs' representations and gradients maintain the same, implying the trained classifier weight to be the same.

} 
\end{proof}
By combining above theorems, we could reach Theorem~\ref{thm:edgebased} by simply adding up the bounded number.

\subsection{Proof of Theorem 4}
Similar to the proof of Theorem~\ref{thm:edgebased}, to prove Theorem~\ref{thm:nodebased}, we first certify the bounded number of altered predictions under (1) edge manipulation, (2) node manipulation and (3) node feature manipulation separately through Theorems~\ref{thm:edgeperturb2}-\ref{thm:nodefeaperturb2}. 

\begin{theorem}[]
\label{thm:edgeperturb2}
Assume $\mathcal{G}_{\text{tr}}$ is under the edge manipulation $\{\mathcal{E}_+,\mathcal{E}_-\}$, then at most $2|\mathcal{E}_+| + 2|\mathcal{E}_-|$  node sub-classifiers trained by our node-centric subgraph sets are different between $\vec{\mathcal{G}}'_{[S]}$ and $\vec{\mathcal{G}}_{[S]}$, and at most $|\mathcal{E}_+| + |\mathcal{E}_-|$ graph sub-classifiers trained by our node-centric subgraph sets are different between $\vec{\mathcal{G}}'_{[S]}$ and $\vec{\mathcal{G}}_{[S]}$.  
\end{theorem}
\label{suppl:proofedge2}
\begin{proof}
For the node classifier, We simply analyze when an arbitrary edge $(u, v)$ is deleted/added from a train graph $G\in\mathcal{G}_{\text{tr}}$. It is obvious at most two subgraphs $\vec{G}_{i_{u \rightarrow v} }$ and $\vec{G}_{i_{v \rightarrow u} }$ are perturbed after perturbation, and therefore two subgraph sets are affected. Generalizing this observation to any $|\mathcal{E}_+| + |\mathcal{E}_-|$ edges in $G$ being perturbed, at most $2|\mathcal{E}_+| + 2|\mathcal{E}_-|$ subgraph sets are generated different between $\mathcal{G}_{[S]}$ and  $\mathcal{G}_{[S]}'$. 

For the graph classifier, we consider the following two cases: i) $i_{u\rightarrow v}=i_{v\rightarrow u}$. this means $u$ and $v$ are in the same subgraph, hence at most one subgraph's representation is affected; ii) $i_{u\rightarrow v}\neq i_{v\rightarrow u}$. Due to the removal of other nodes whose subgraph index is not $i$ in every subgraph $\vec{G}_{i}$, both direct edges would always be removed from $\vec{G}_{i_{u \rightarrow v} }$ and $\vec{G}_{i_{v \rightarrow u} }$ if exist. Generalizing this observation to any $|\mathcal{E}_+| + |\mathcal{E}_-|$ edges in $G$ being perturbed, at most $|\mathcal{E}_+| + |\mathcal{E}_-|$ subgraph sets are generated different between $\mathcal{G}_{[S]}$ and  $\mathcal{G}_{[S]}'$.  

\vspace{-2mm}
\end{proof}

\begin{theorem}[]
\label{thm:nodeperturb2} 
Assume a graph $G$ is under the node manipulation  
$\{\mathcal{V}_+, \mathcal{E}_{\mathcal{V}_+}, {\bf X}'_{\mathcal{V}_+}, \mathcal{V}_-, \mathcal{E}_{\mathcal{V}_-}\}$,
then at most $|\mathcal{V}_+| + |\mathcal{V}_-| $ node/graph sub-classifiers trained by our node-centric subgraph sets are different between $\vec{\mathcal{G}}'_S$ and $\vec{\mathcal{G}}_S$. 
\end{theorem}

\begin{theorem}[]
\label{thm:nodefeaperturb2} 
Assume a graph $G$ is under the node feature manipulation 
$\{\mathcal{V}_r, \mathcal{E}_{\mathcal{V}_r},{\bf X}'_{\mathcal{V}_r}\}$, 
then at most $|\mathcal{V}_r|$ node/graph sub-classifiers trained by our node-centric subgraphs are different between $\vec{\mathcal{G}}'_S$ and $\vec{\mathcal{G}}_S$. 
\end{theorem}

\label{suppl:proofnode&fea2}
\begin{proof}
Our proof for the above two theorems is based on the key observation that: \emph{in a directed graph, manipulations on nodes with no outgoing edge have no influence on other nodes' representations in GNNs}. For any node  $u \in G$, only one subgraph $\vec{G}_{h[\mathrm{str}(u)] \, \, \texttt{mod} \, \, S+1}$ has outgoing edges.
Take node injection for instance and the proof for other cases are similar. Note that all subgraphs after node injection will contain newly injected nodes $V_{+}$, but they still do not have overlapped nodes with outgoing edges between each other via the hashing mapping. Hence, the injected nodes only have outgoing edges in at most $|V_{+}|$ subgraphs. 
Due to the directed message passing mechanism in GNNs, every node only uses its incoming neighboring nodes' representation to update its own representation. Hence, 
the injected nodes with no outgoing edges, whatever their features ${\bf X}'_{\mathcal{V}_+}$ are, would have no influence on other nodes' representation and gradients, including the training nodes', implying at least $S-|V_{+}|$ subgraphs' training process maintain the same.
\vspace{-2mm}
\end{proof}
}
By collaborating above theorems together, we could reach Theorem~\ref{thm:nodebased} by simply adding up the bounded number.

\begin{figure*}[t]
    \centering
    \captionsetup[subfloat]{labelsep=none, format=plain, labelformat=empty}

    \subfloat[{\footnotesize (a) Edge-Centric Graph Division for Node Classification against edge deletion, node deletion and node feature manipulation}]{
    \includegraphics[width=0.9\linewidth]{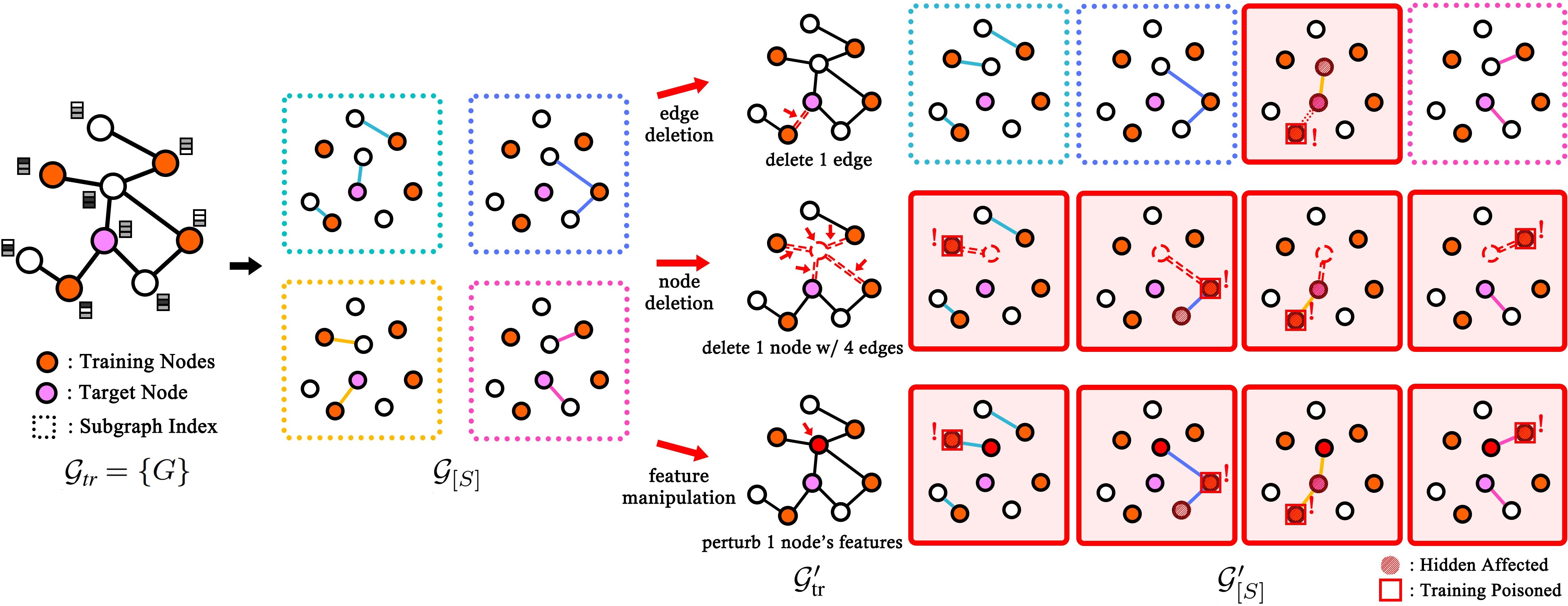}}
    \hspace{+10mm}
   
    \subfloat[{\footnotesize (b) Node-Centric Graph Division for Node Classification against edge deletion, node deletion and node feature manipulation}]{
    \includegraphics[width=0.9\linewidth]{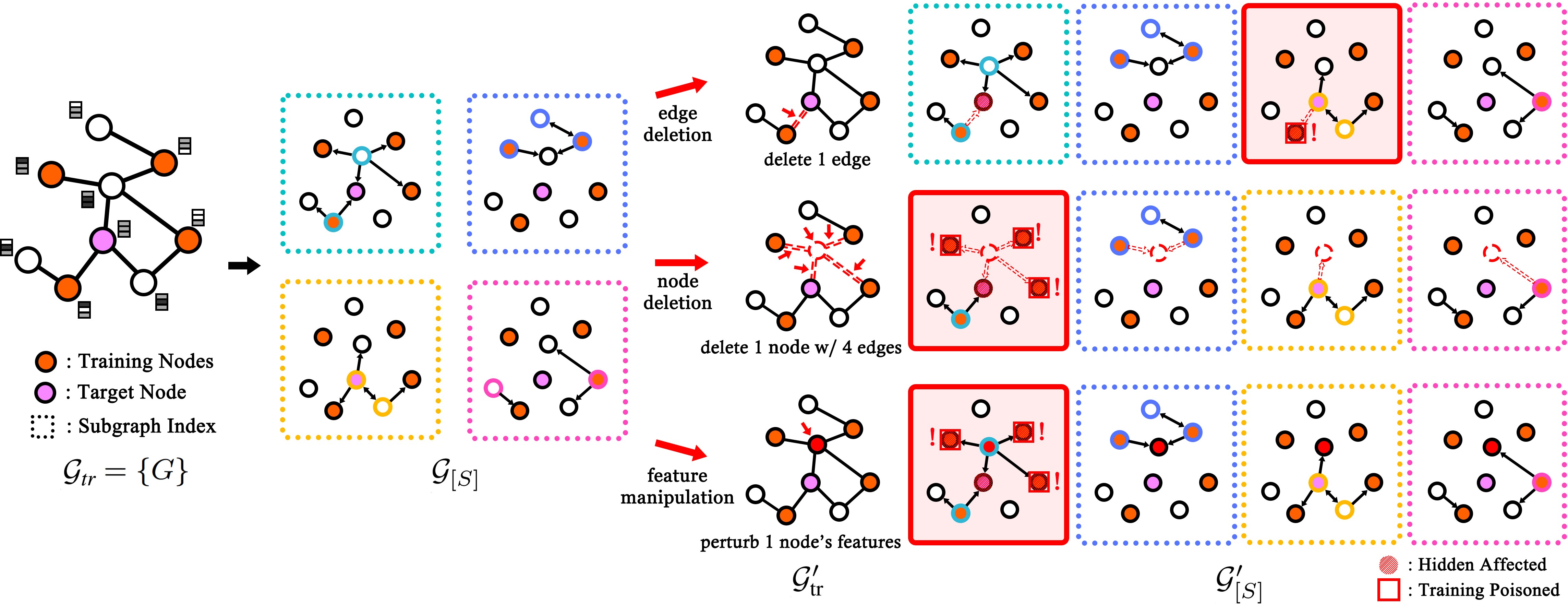}}
    \hspace{+10mm}
    \caption{Illustration of our edge-centric and node-centric graph division strategies for node classification against edge deletion, node deletion, and node feature manipulation. 
    {\bf To summarize:} 1 deleted edge  affects at most 1 subgraph prediction in both graph division strategies. In contrast, 1 deleted node with, e.g., $3$ incident edges can affect at most 3 subgraph predictions with edge-centric graph division, but at most 1 subgraph prediction with node-centric graph division.
    }
    \label{fig:subgraphs_NC_more}
   \vspace{-2mm}
\end{figure*}

\begin{figure*}[t]
    \centering
    \captionsetup[subfloat]{labelsep=none, format=plain, labelformat=empty}

    \subfloat[{\footnotesize (a) Edge-Centric Graph Division for Graph Classification against edge manipulation, node manipulation and feature manipulation}]{
    \includegraphics[width=0.9\linewidth]{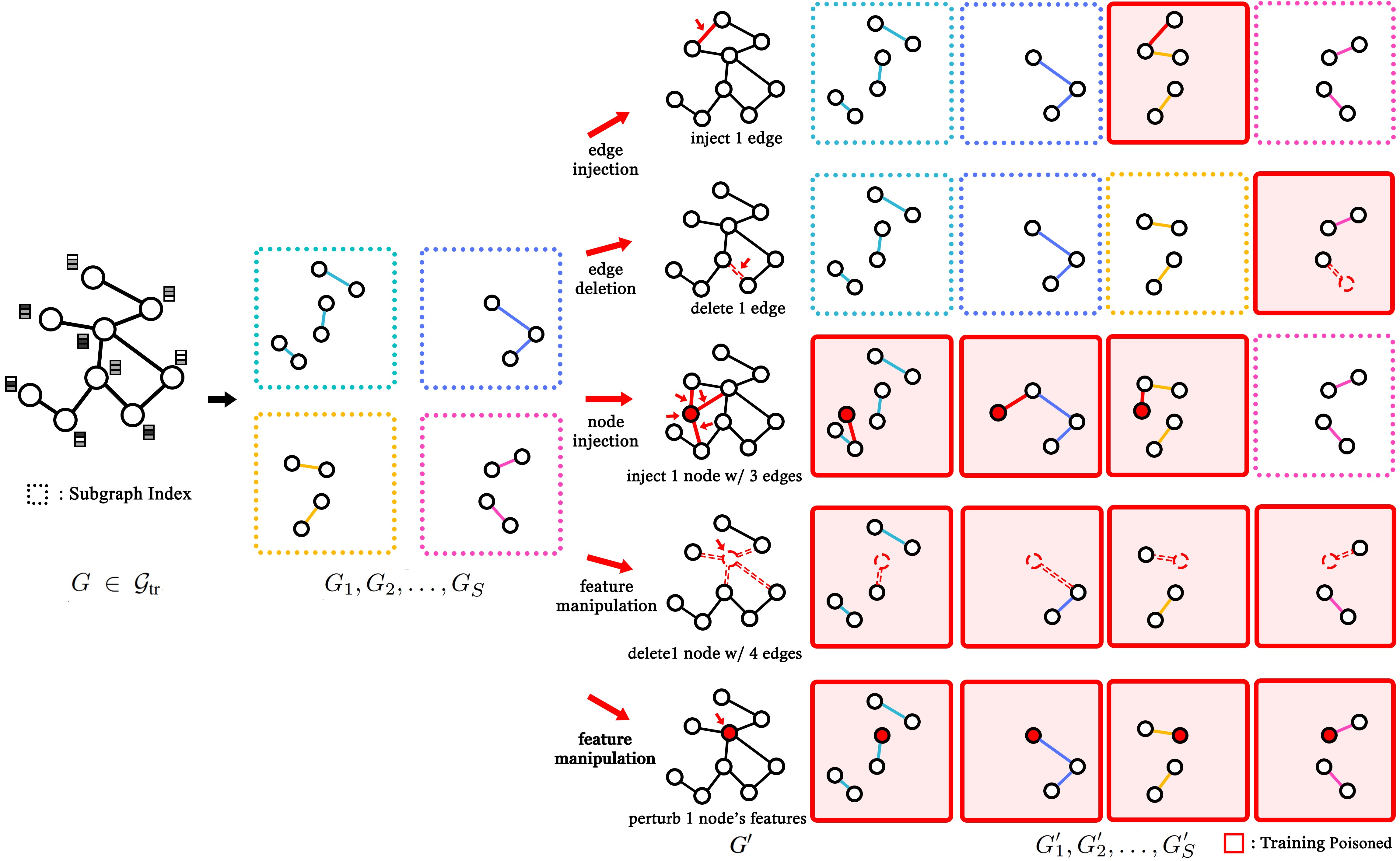}}
    \hspace{+20mm}
    
    \subfloat[{\footnotesize (b) Node-Centric Graph Division for Graph Classification against edge manipulation, node manipulation and feature manipulation}]{
    \includegraphics[width=0.9\linewidth]{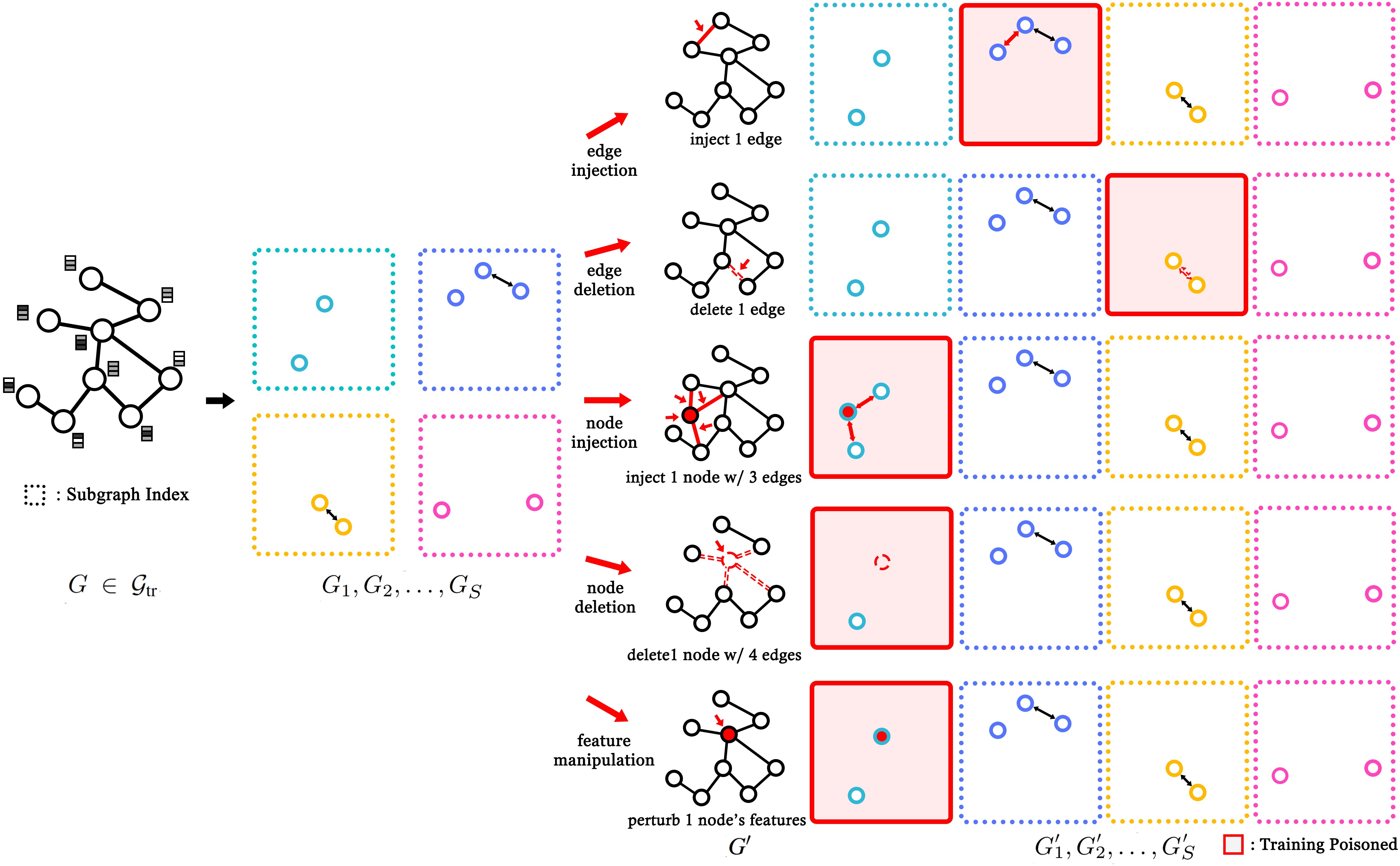}}
    \vspace{-2mm}
    \caption{Illustration of our edge-centric and node-centric graph division strategies for graph classification. The conclusion are similar to those for node classification.}
    \label{fig:subgraphs_GC}
\end{figure*}

\begin{table}[!t]
    \footnotesize
    \centering 
    \renewcommand\arraystretch{1.3}
    \begin{tabular}{c|c|c|c|c}
     \toprule
          {\bf Node Classification}&{\bf Ave degree}&{$|\mathcal{V}|$}&$|\mathcal{E}|$&$|\mathcal{C}|$ \\
         \Xhline{0.9pt}
       Cora-ML&5.6&2, 995&8,416&7\\
         \cline{1-5} 
         Citeseer&2.8&3,327&4,732&6\\
         \cline{1-5} 
         Pubmed&4.5&19,717&44,338&3\\
         \cline{1-5} 
         Amazon-C&71.5&13,752&491,722&10\\
         \Xhline{1.2pt} 
       {\bf Graph Classification}&$|\mathcal{G}|$&$|\mathcal{V}|_{avg}$&$|\mathcal{E}|_{avg}$&$|\mathcal{C}|$ \\
         \Xhline{0.9pt}
         {AIDS}&2,000&15.7&16.2&2\\
         \cline{1-5} 
        MUTAG&4,337&30.3&30.8&2\\
         \cline{1-5} 
         PROTEINS&1,113&39.1&72.8&2\\
         \cline{1-5} 
        DD&1,178&284.3&715.7&2\\
       \bottomrule
    \end{tabular}
    \caption{Datasets and their statistics.}
    \label{tab:datasets}
\end{table}

\begin{figure*}[!t]
\centering
\subfloat[Cora-ML]{\includegraphics[width=0.25\textwidth]{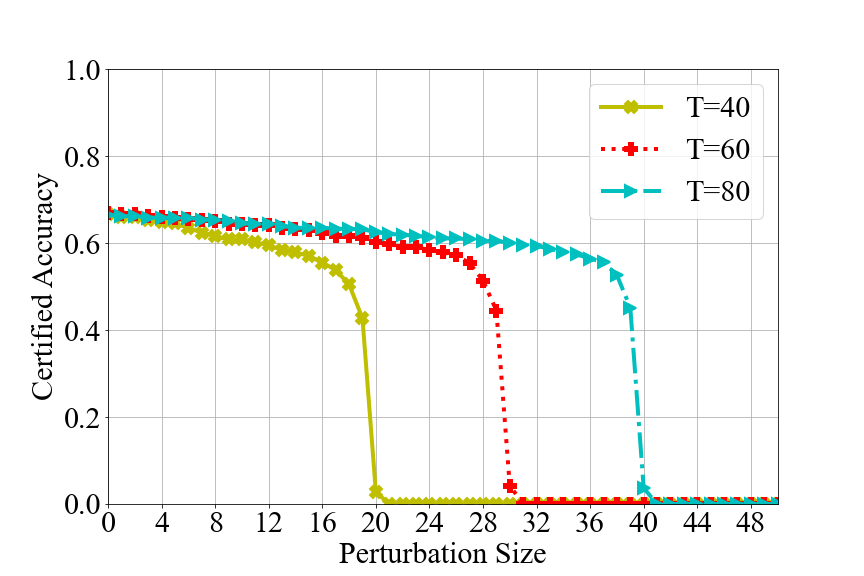}}\hfill
\subfloat[Citeseer]{\includegraphics[width=0.25\textwidth]{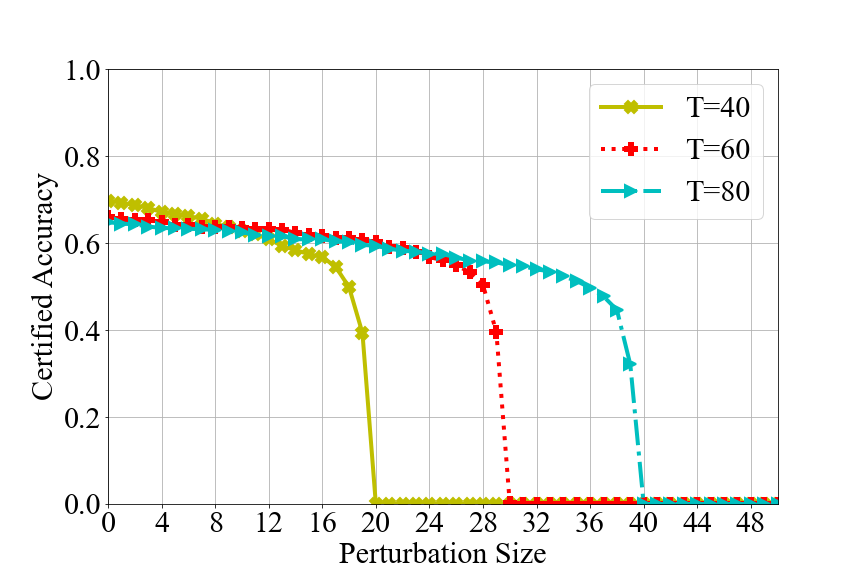}}\hfill
\subfloat[Pubmed]{\includegraphics[width=0.25\textwidth]{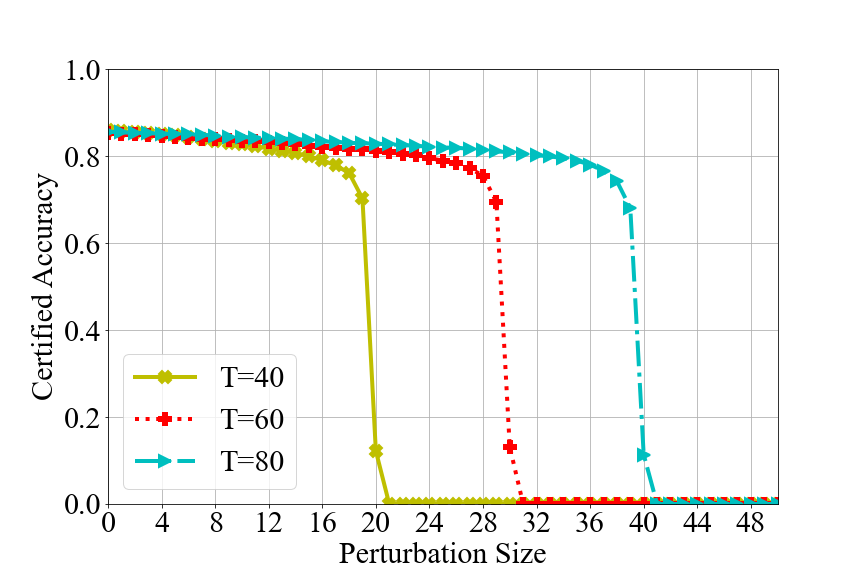}}\hfill
\subfloat[Amazon-C]{\includegraphics[width=0.25\textwidth]{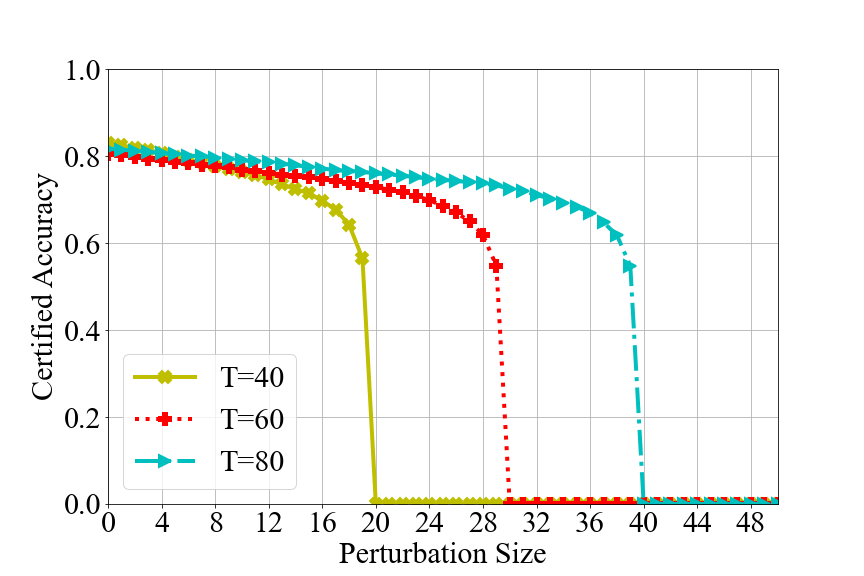}}\\
\caption{Certified node accuracy of our {\nameE} with GSAGE w.r.t. the number of subgraphs $S$.}
\label{fig:node-EC-T-GSAGE}
\end{figure*}

\begin{figure*}[!t]
\centering
\subfloat[Cora-ML]{\includegraphics[width=0.25\textwidth]{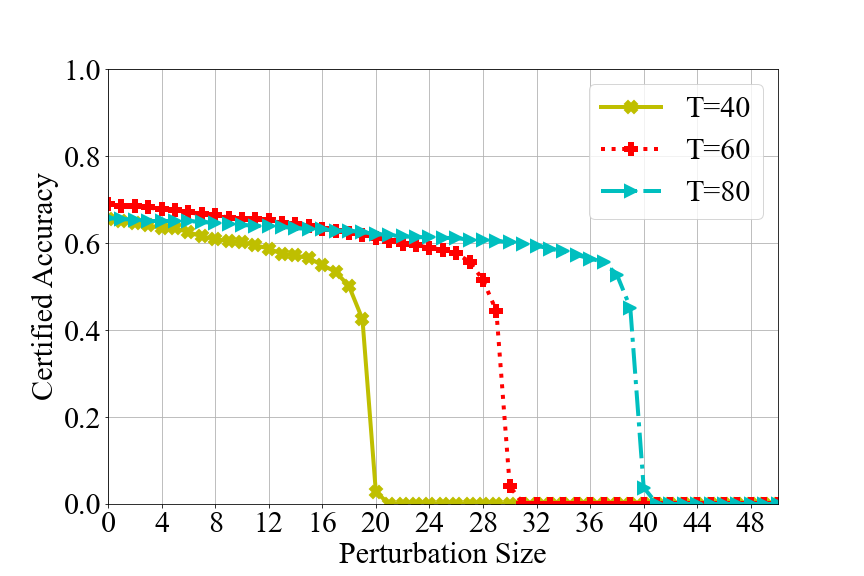}}\hfill
\subfloat[Citeseer]{\includegraphics[width=0.25\textwidth]{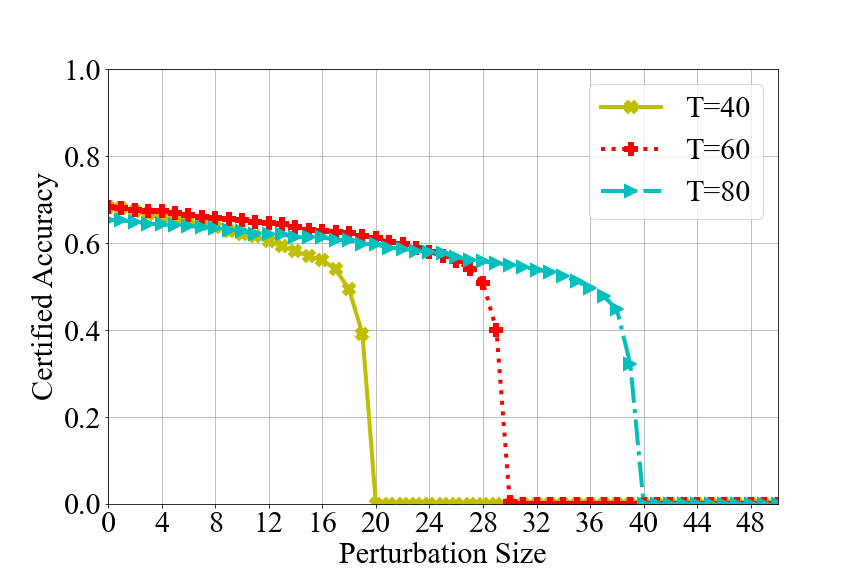}}\hfill
\subfloat[Pubmed]{\includegraphics[width=0.25\textwidth]{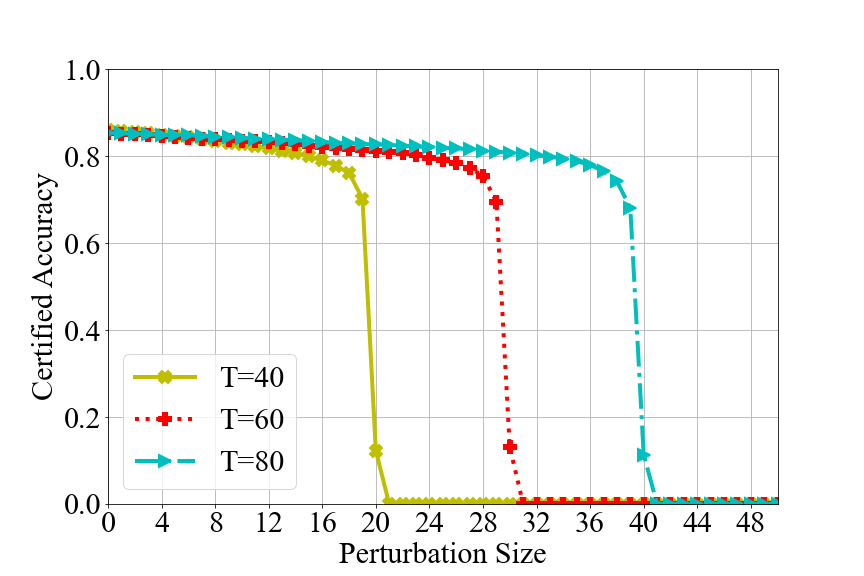}}\hfill
\subfloat[Amazon-C]{\includegraphics[width=0.25\textwidth]{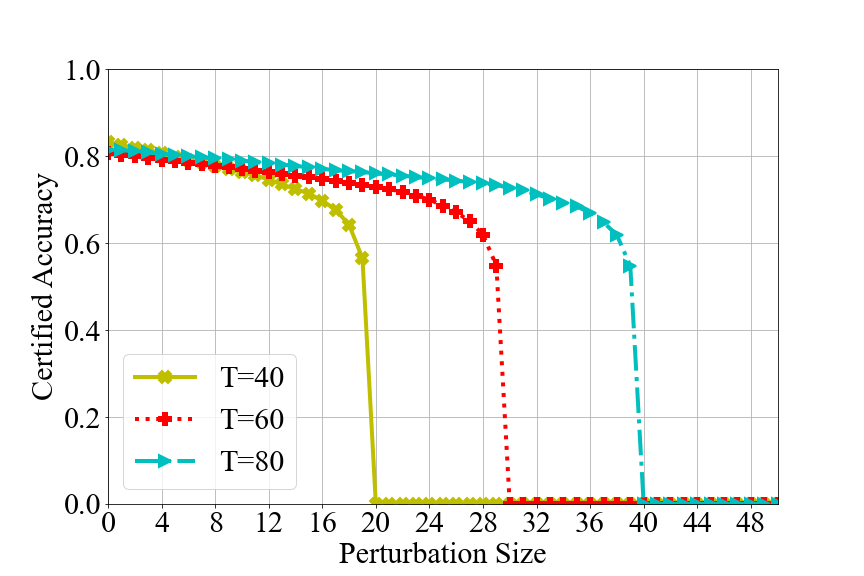}}\\
\caption{Certified node accuracy of our {\nameN} with GSAGE w.r.t. the number of subgraphs $S$.}
\label{fig:node-NC-T-GSAGE}
\end{figure*}

\begin{figure*}[!t]
\centering
\subfloat[Cora-ML]{\includegraphics[width=0.25\textwidth]{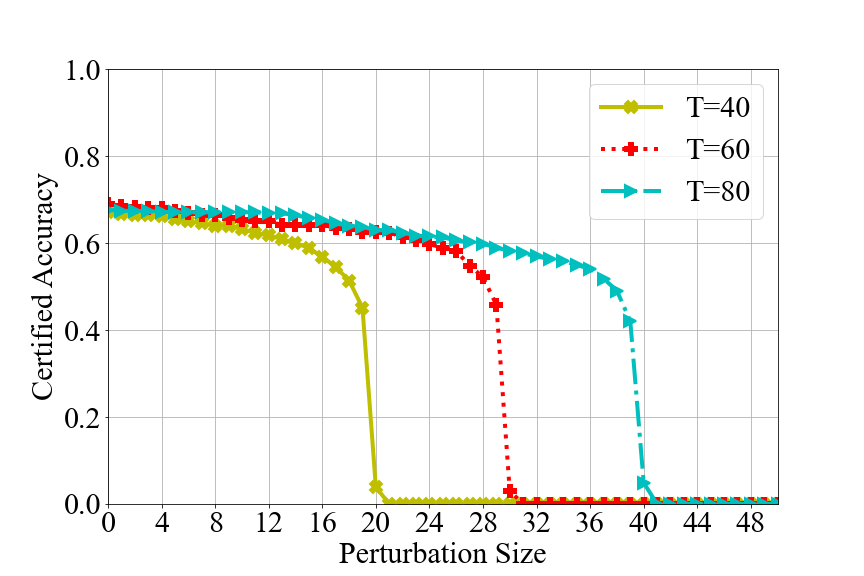}}\hfill
\subfloat[Citeseer]{\includegraphics[width=0.25\textwidth]{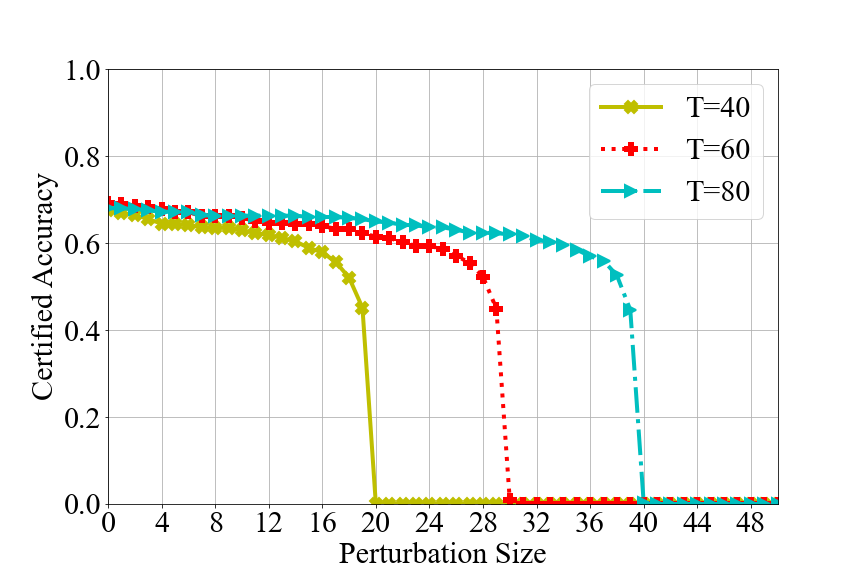}}\hfill
\subfloat[Pubmed]{\includegraphics[width=0.25\textwidth]{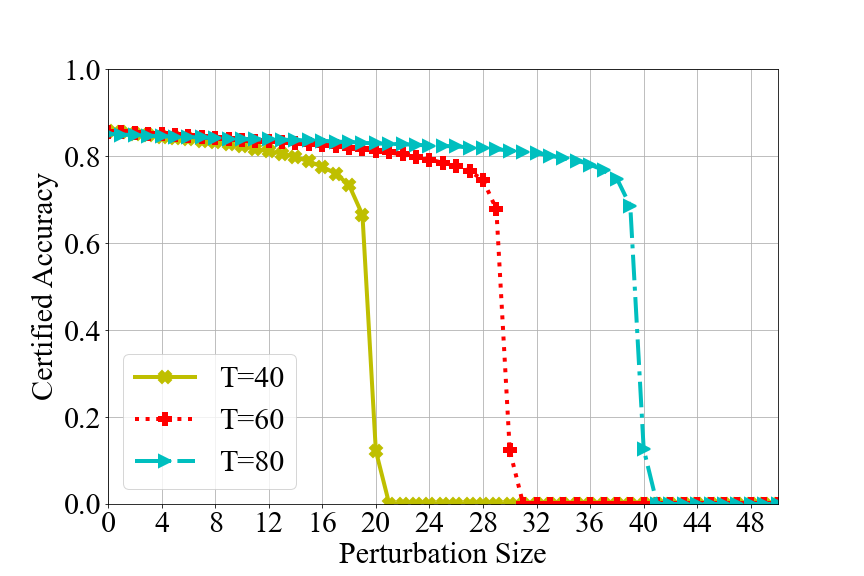}}\hfill
\subfloat[Amazon-C]{\includegraphics[width=0.25\textwidth]{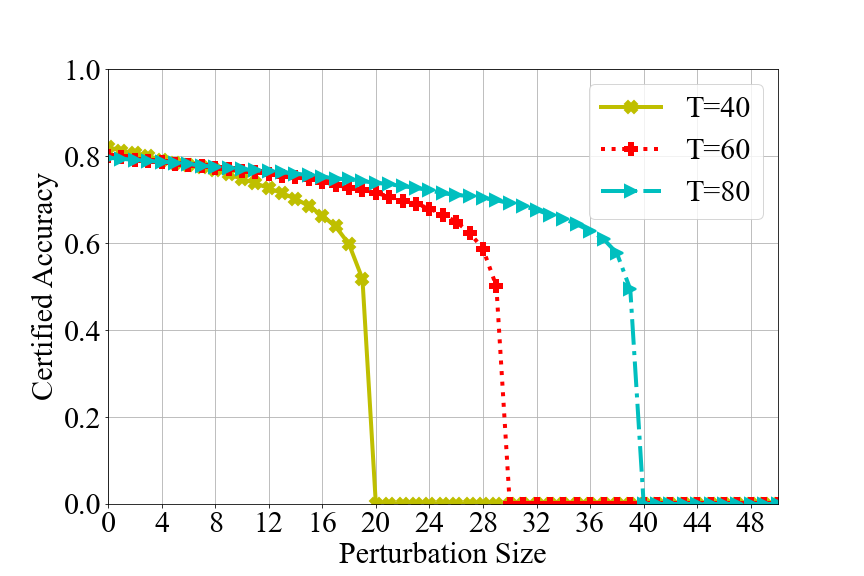}}\\
\caption{Certified node accuracy of our {\nameE} with GAT w.r.t. the number of subgraphs $S$.}
\label{fig:node-EC-T-GAT}
\end{figure*}

\begin{figure*}[!t]
\centering
\subfloat[Cora-ML]{\includegraphics[width=0.25\textwidth]{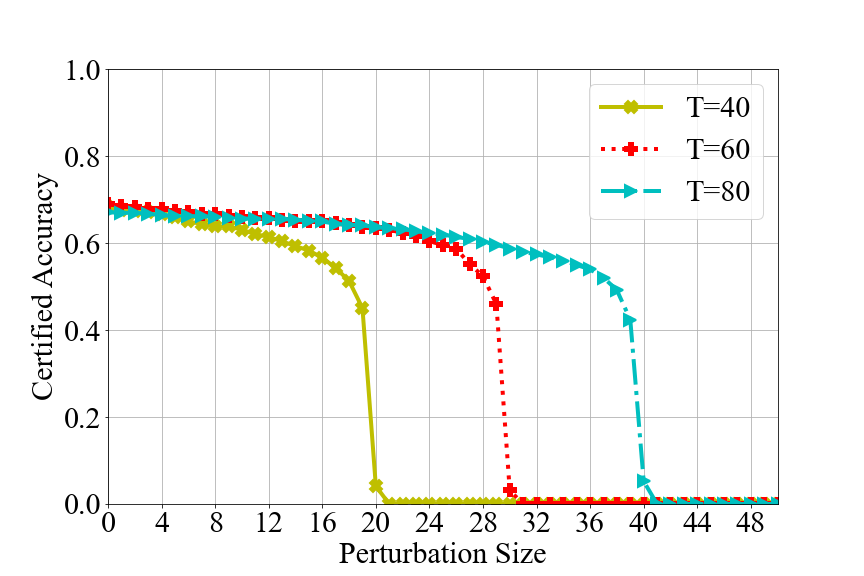}}\hfill
\subfloat[Citeseer]{\includegraphics[width=0.25\textwidth]{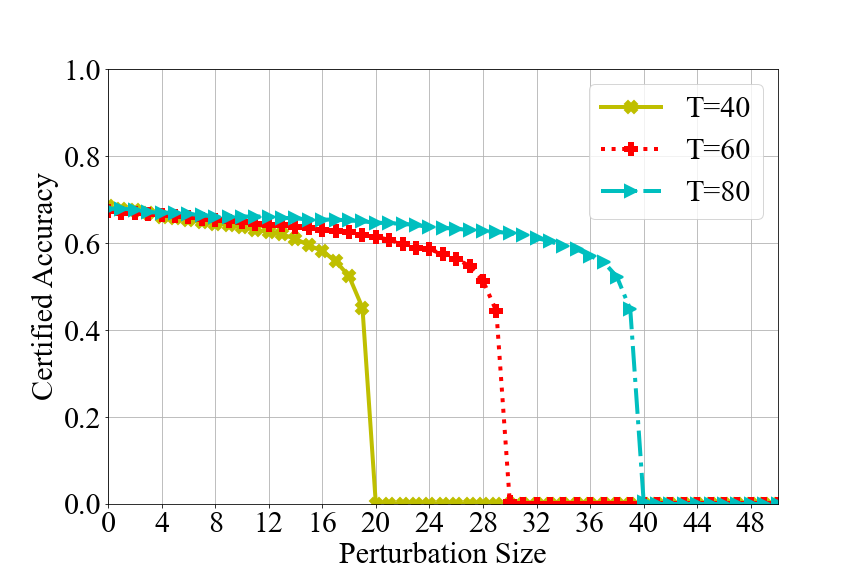}}\hfill
\subfloat[Pubmed]{\includegraphics[width=0.25\textwidth]{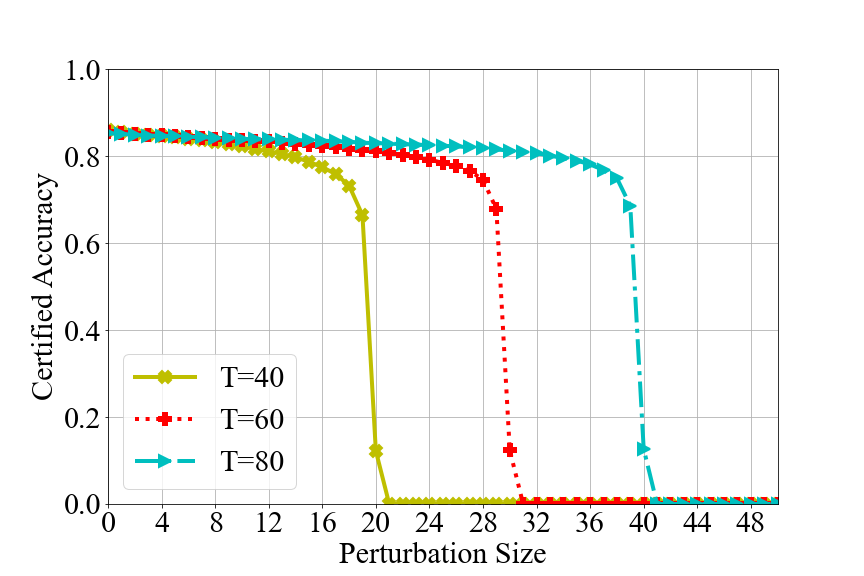}}\hfill
\subfloat[Amazon-C]{\includegraphics[width=0.25\textwidth]{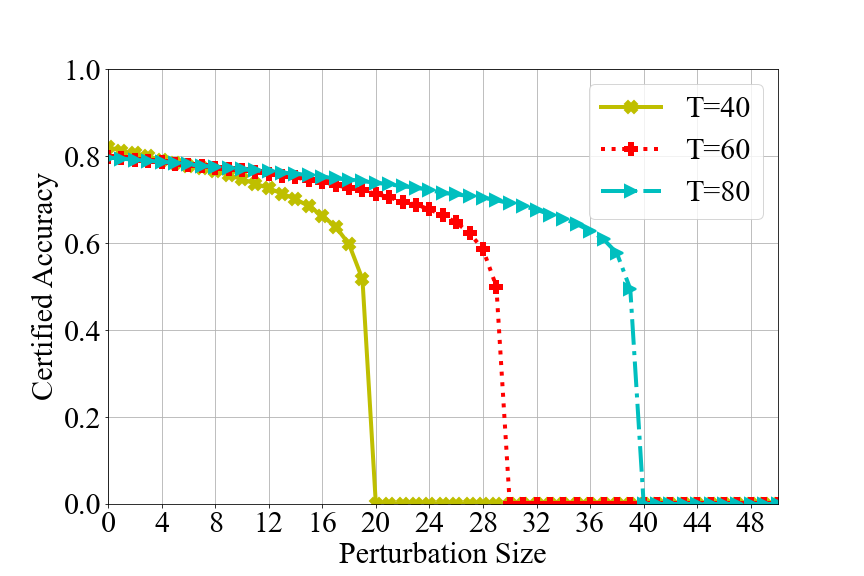}}\\
\caption{Certified node accuracy of our {\nameN} with GAT w.r.t. the number of subgraphs $S$.}
\label{fig:node-NC-T-GAT}
\end{figure*}

\begin{figure*}[!t]
\centering
\subfloat[Cora-ML]{\includegraphics[width=0.25\textwidth]{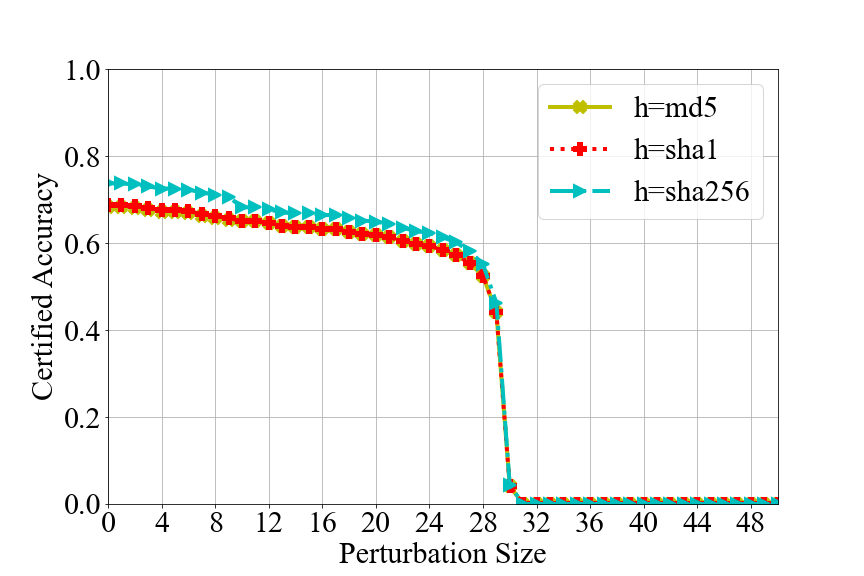}}\hfill
\subfloat[Citeseer]{\includegraphics[width=0.25\textwidth]{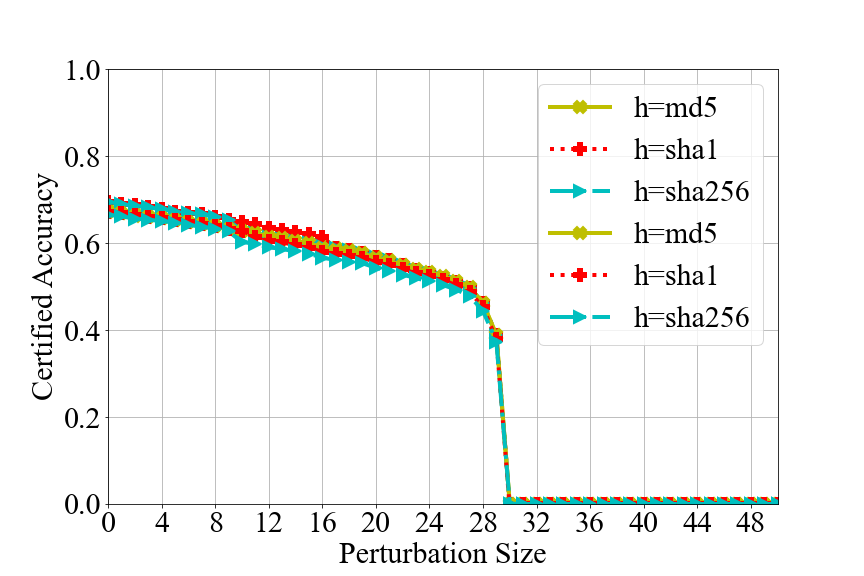}}\hfill
\subfloat[Pubmed]{\includegraphics[width=0.25\textwidth]{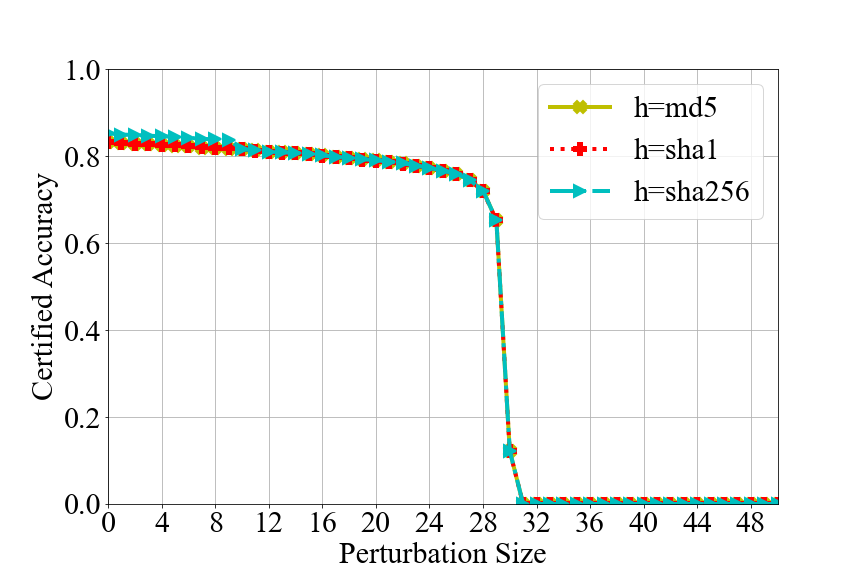}}\hfill
\subfloat[Amazon-C]{\includegraphics[width=0.25\textwidth]{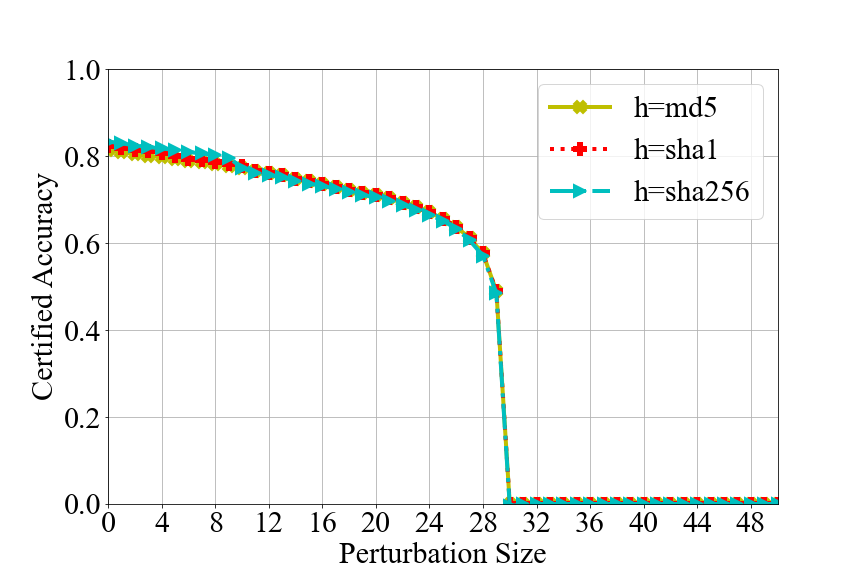}}\\
\caption{Certified node accuracy of our {\nameE} w.r.t. the hash function $h$.}
\label{fig:node-EC-hash}
\end{figure*}

\begin{figure*}[!t]
\centering
\subfloat[Cora-ML]{\includegraphics[width=0.25\textwidth]{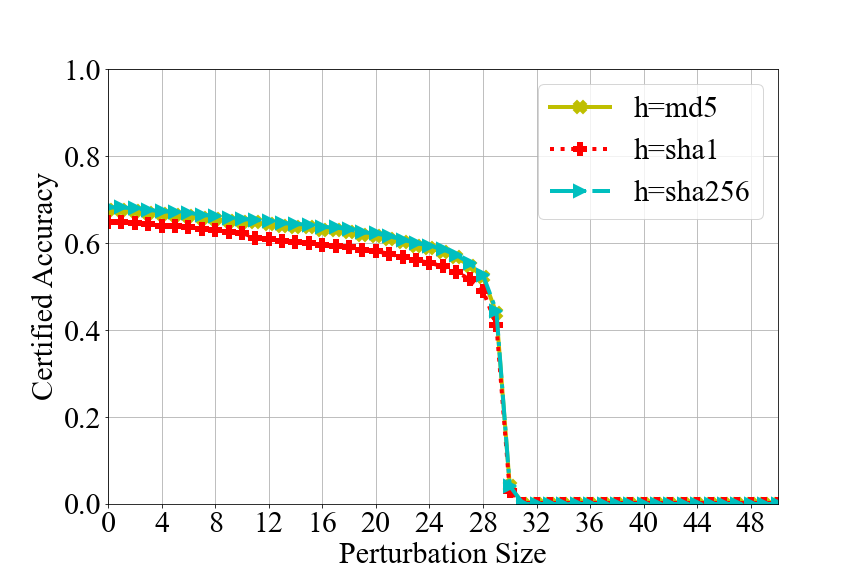}}\hfill
\subfloat[Citeseer]{\includegraphics[width=0.25\textwidth]{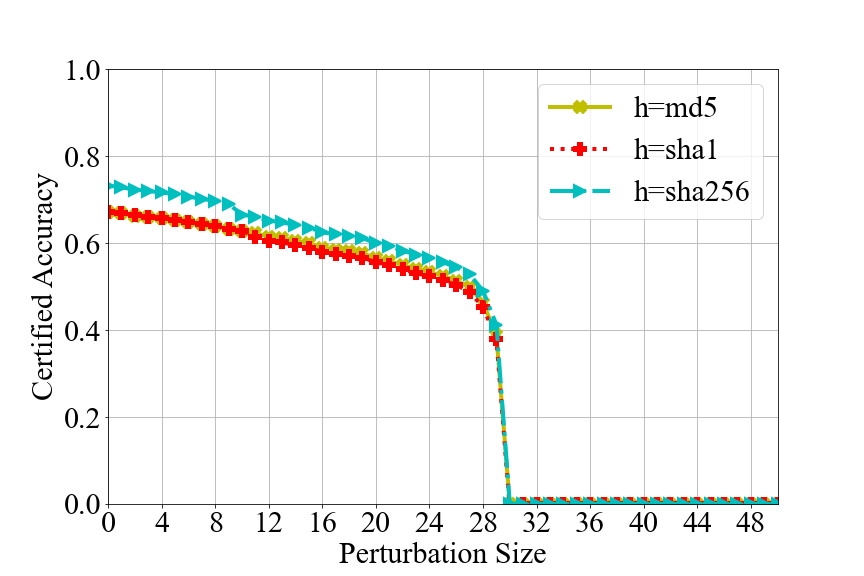}}\hfill
\subfloat[Pubmed]{\includegraphics[width=0.25\textwidth]{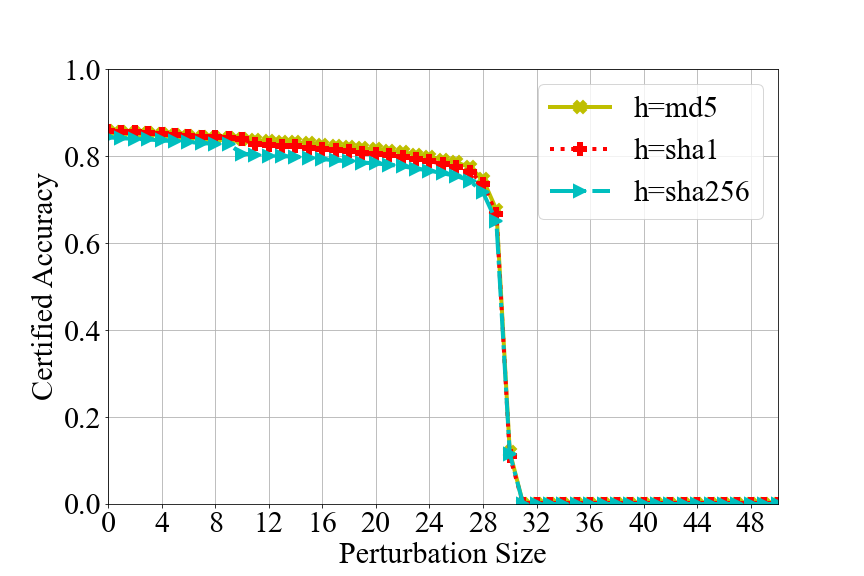}}\hfill
\subfloat[Amazon-C]{\includegraphics[width=0.25\textwidth]{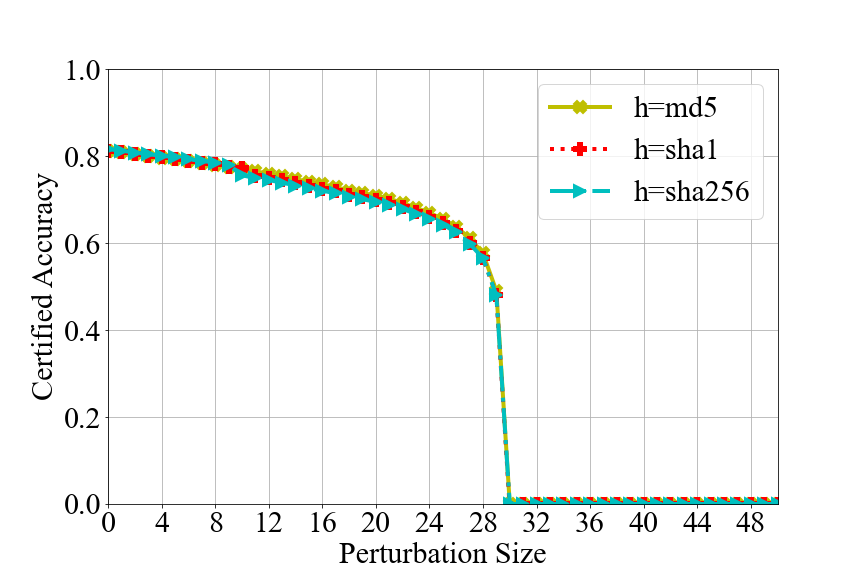}}\\
\caption{Certified node accuracy of our {\nameN} w.r.t. the hash function $h$.}
\label{fig:node-NC-hash}
\end{figure*}

\begin{figure*}[!t]
\centering
\subfloat[AIDS]{\includegraphics[width=0.25\textwidth]{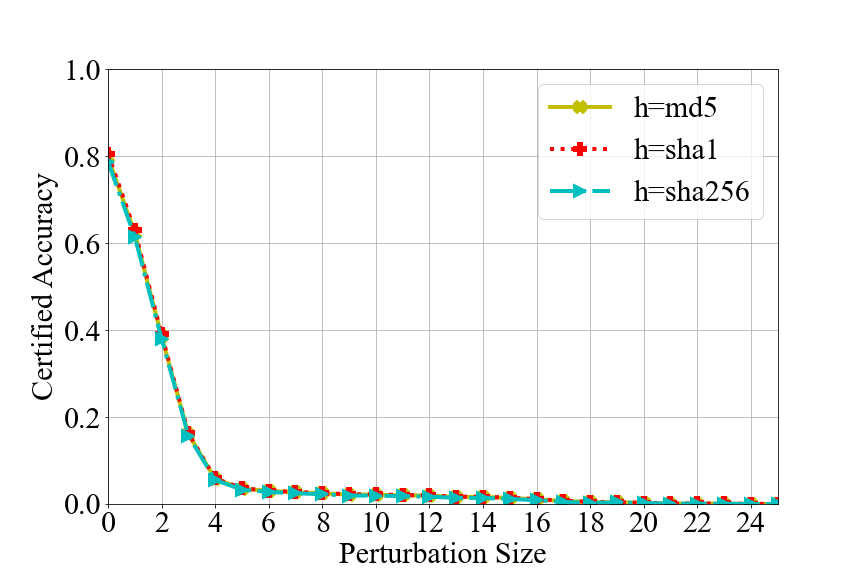}}\hfill
\subfloat[MUTAG]{\includegraphics[width=0.25\textwidth]{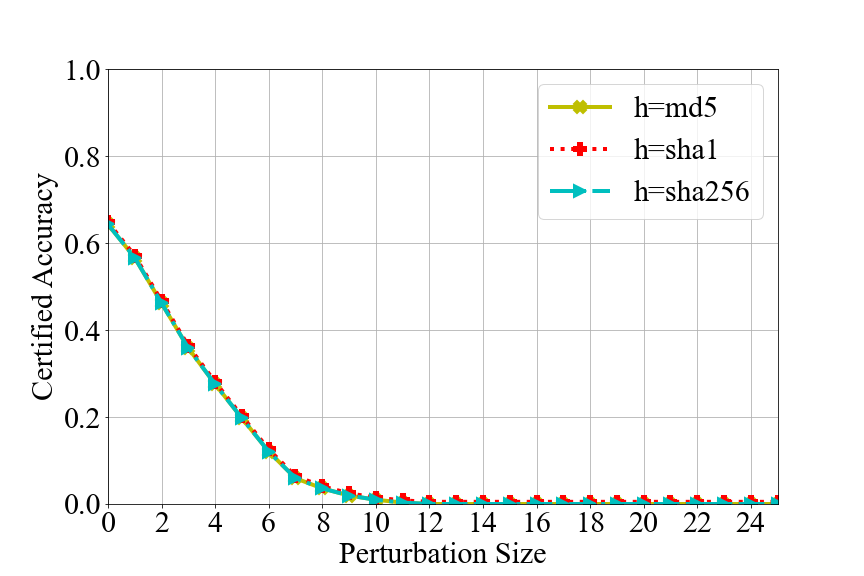}}\hfill
\subfloat[PROTEINS]{\includegraphics[width=0.25\textwidth]{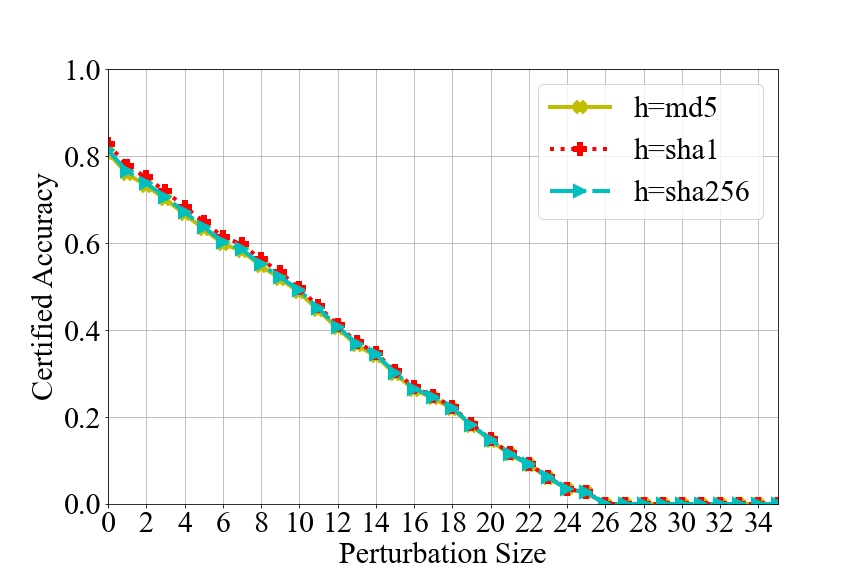}}\hfill
\subfloat[DD]{\includegraphics[width=0.25\textwidth]{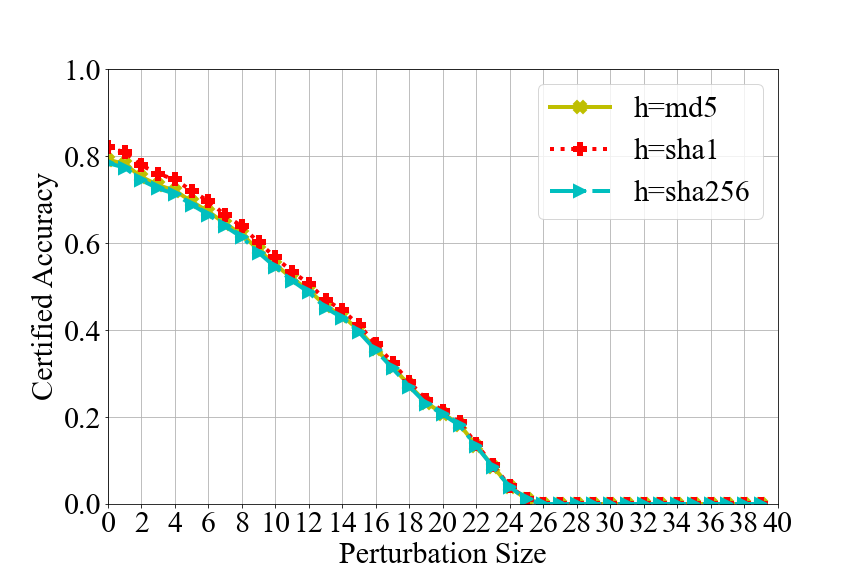}}\\
\caption{Certified graph accuracy of our {\nameE} w.r.t. the hash function $h$.}
\label{fig:graph-EC-hash}
\end{figure*}

\begin{figure*}[!t]
\centering
\subfloat[AIDS]{\includegraphics[width=0.25\textwidth]{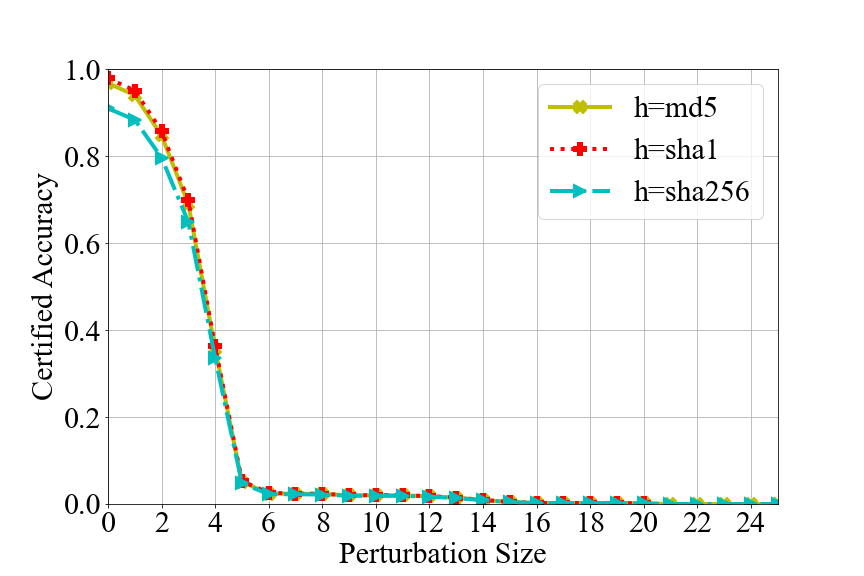}}\hfill
\subfloat[MUTAG]{\includegraphics[width=0.25\textwidth]{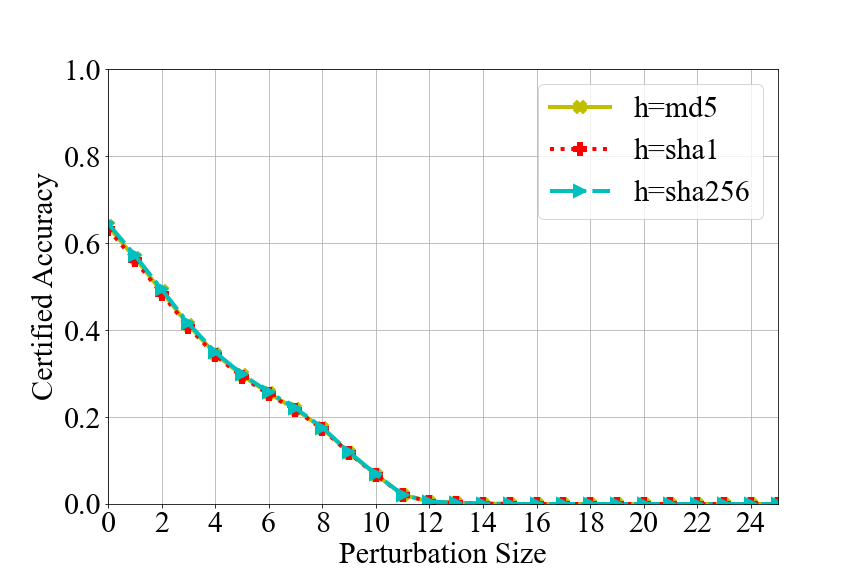}}\hfill
\subfloat[PROTEINS]{\includegraphics[width=0.25\textwidth]{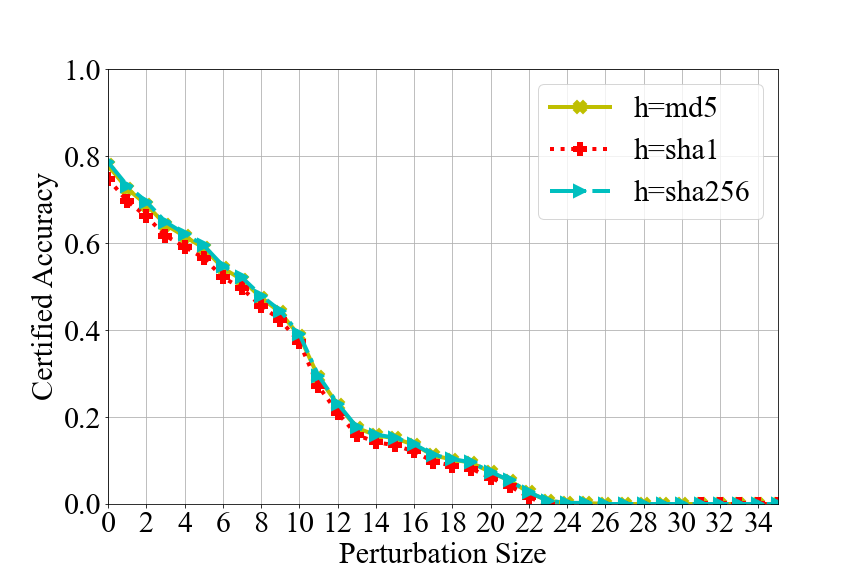}}\hfill
\subfloat[DD]{\includegraphics[width=0.25\textwidth]{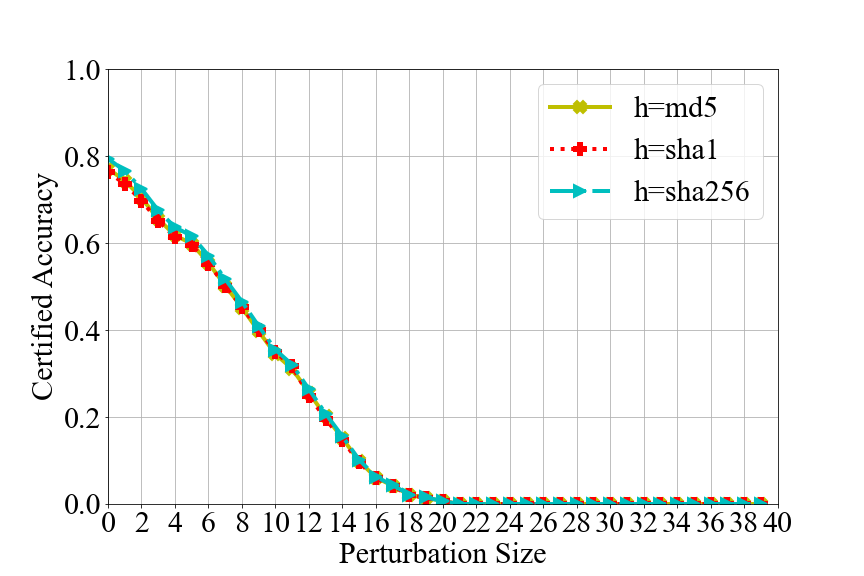}}\\
\caption{Certified graph accuracy of our {\nameN} w.r.t. the hash function $h$.}
\label{fig:graph-NC-hash}
\vspace{-2mm}
\end{figure*}

\end{document}